\documentclass[11pt, a4paper,english]{article}
 \usepackage{amsmath,amssymb}
 \usepackage{bm}
 \usepackage{ascmac}
 \usepackage{amsthm}
 \usepackage{enumerate}
 \usepackage[dvips]{color}
 \usepackage{here}
\usepackage[pdftex]{graphicx}
\usepackage{mathrsfs}

\theoremstyle{definition}
\newtheorem{theorem}{Theorem}[section]

\newtheorem{lemma}{Lemma}[section]
\newtheorem{cor}{Corollary}[section]

\newtheorem{definition}{Definition}[section]

\makeatletter 
 \def\section{\@startsection {section}{1}{\z@}{3.5ex plus -1ex minus -.2ex}{2.3 ex plus .2ex}{\bf}}
 \def\@seccntformat#1{\csname the#1\endcsname.\ }
 \makeatother
 
  \makeatletter 
 \def\subsection{\@startsection {subsection}{1}{\z@}{3.5ex plus -1ex minus -.2ex}{2.3 ex plus .2ex}{\bf}}
 \def\@seccntformat#1{\csname the#1\endcsname.\ }
 \makeatother

\numberwithin{equation}{section} 

\newcommand{\R}{\mathbb{R}}

\newcommand{\V}{{\rm{Var}}}
\newcommand{\Cov}{{\rm{Cov}}}
\newcommand{\PR}{{\rm{P}}}
\newcommand{\q}{\quad}

\newcommand{\f}{^\forall}

\newcommand{\argmax}{\operatornamewithlimits{argmax}}

\newcommand{\bx}{{\bm{x}}}
\newcommand{\ve}{\varepsilon}

\usepackage{algorithm}
\usepackage{algcompatible}

\usepackage{slashbox}

\usepackage{geometry}
\title{Active learning for enumerating local minima based on Gaussian process  derivatives} 
\author{Yu Inatsu  \thanks{RIKEN Center for Advanced Intelligence Project} \and Daisuke sugita \thanks{Nagoya Institute of Technology} \and Kazuaki Toyoura \thanks{Kyoto University} \and  Ichiro Takeuchi \footnotemark[1] \footnotemark[2] \thanks{Center for Materials Research by Information Integration, National Institute for Materials Science} \thanks{E-mail:takeuchi.ichiro@nitech.ac.jp}}
\date{\today}

\begin{document}
\maketitle

\begin{abstract}
We study active learning (AL) based on Gaussian Processes (GPs) for efficiently enumerating all of the local minimum solutions of a black-box function. This problem is challenging due to the fact that local solutions are characterized by their zero gradient and positive-definite Hessian properties, but those derivatives cannot be directly observed. We propose a new AL method in which the input points are sequentially selected such that the confidence intervals of the GP derivatives are effectively updated for enumerating local minimum solutions. We theoretically analyze the proposed method and demonstrate its usefulness through numerical experiments.
\end{abstract}


\section{Introduction}
In many areas of science and technology, machine learning has been successfully used for uncovering unknown complex systems which are formulated as black-box functions. 
When the evaluation of a black-box function is expensive, it is often difficult to exhaustively investigate the function in the entire input domain.
\emph{Active learning (AL)} has been developed as a method for effectively selecting the input points at which the function evaluations are helpful for the target task.
For example, if the target task is to find the global minimum, it is reasonable to evaluate the function at the input points which are likely to be global minima (this AL problem has been intensively studied in the context of \emph{Bayesian Optimization (BO)}~\cite{settles2009active,NIPS2011_4221,NIPS2007_3189,pmlr-v77-nguyen17a,Wang:2016:BOB:3013558.3013569,Hennig:2012:ESI:2503308.2343701}).

In this paper, we study the problem of enumerating local minima (or maxima) of a black-box function.
In many applications, it is beneficial to identify the positions of local minima and/or maxima because it helps to roughly grasp the ``shape'' of the black-box function. 
Furthermore, it is often the case that each local minimum point has its own special meaning. 
For example, when modeling the energy space of a physical system, each local minimum point corresponds to a stable energy point of the system, which is crucially important for revealing various physical properties of the system (see \S5 for an application of the proposed method to a physical problem).

A local minimum point is characterized by the first and the second derivatives of the function, i.e., an input point is a local minimum if the gradient vector is zero and the Hessian matrix is \emph{positive-definite (PD)}.
The difficulty of this problem is due to the fact that we need to select the input points which are likely to be local minima under a situation that those derivatives cannot be directly observed.
In other words, we need to select a set of input points at which the function evaluations are helpful for getting information on the zero gradient and the PD Hessian properties. 

We employ \emph{Gaussian Processes (GPs)} for modeling a black-box function.
GPs are useful in many AL problems since they enable one to predict not only the average but also the uncertainty of the black-box function. 
Our basic idea is to exploit the property that the derivative of a GP is also a GP. 
Based on this property, we develop a method for computing the confidence intervals (CIs) of each element of the gradient vector and the minimum eigenvalue of the Hessian matrix\footnote{A Hessian matrix is PD if and only if the minimum eigenvalue is positive.}.
Then, these CIs are used for designing an \emph{acquisition function (AF)} for efficiently enumerating all of the local minima.
We call the proposed method  \emph{Active learning for Local Optima Enumeration (ALOE)}. 

\paragraph{Related works}
BO has been intensively studied  (see \cite{shahriari2016taking,settles2009active} for comprehensive survey of BO).
In a few existing BO studies, the gradient of a GP is used for accelerating the BO task. 
For example, \cite{NIPS2017_7111} discussed the advantage of using the gradient in a framework called \emph{Knowledge-Gradient}.
Furthermore, \cite{NIPS2002_2287} demonstrated that the gradient of a GP is helpful for modeling dynamical systems. 
In these  works, it is assumed that not only the function values but also the gradient vectors are directly observed. 
On the other hand, we consider a setup where neither the gradient nor Hessian are directly observed.  
The CI-based approach in ALOE is motivated by \cite{Gotovos:2013:ALL:2540128.2540322}, in which the CIs of function values are used for estimating a level set of the function.
Similarly, the CIs of the function values were also used for safe BO in \cite{sui2015safe}.
We employ some of the theoretical techniques developed in \cite{Srinivas:2010:GPO:3104322.3104451,Gotovos:2013:ALL:2540128.2540322} for analyzing the various theoretical properties of ALOE.
In contrast to these existing studies, we use the CIs of the gradient and the Hessian, which are not be easily available since they cannot be directly observed. 

\paragraph{Our contribution}
To the best of our knowledge, there is no existing AL method for enumerating local minima.
We propose a new AL method called ALOE, in which we develop a method to compute the CIs of the gradients and the minimum eigenvalue of the Hessian, without observing these derivatives directly.
Furthermore, based on these CIs, we propose a novel AF for efficiently enumerating local minima. 
We theoretically analyze the accuracy and the convergence of ALOE, and evaluate its empirical performance by numerical experiments with synthetic data and real application to a physical problem.

\section{Preliminaries}
\paragraph{Problem setting}
Suppose that an unknown function $f: D \to \R$ is defined on a set $D \subseteq \R^d$.
%
%
For simplicity, we consider a finite set of input points $\mathcal{X} \subset D$, and consider an AL method to classify if each point $\bm x \in \mathcal{X}$ is local minimum point\footnote{All the methods and theories in this paper can be extended to the case where $\mathcal{X}$ is continuous  with reasonable  assumptions.}.
Let us define the following subset of points in $\mathcal{X}$. 
\begin{definition}[The set of local minima]
\begin{align}
 \label{eq:localMinima}
 S :=
 \left\{
 \bm x \in \mathcal{X} ~\Big|~ 
 \frac{\partial f}{\partial \bm x} = \bm 0
 \text{ and }
 \frac{\partial^2 f}{\partial \bm x \bm x^\top} 
 \succ 0
 \right\},
\end{align}
where
$M \succ 0$
indicates that the matrix $M$ is PD.
\end{definition}
Note that $S$ does not contain ``pathological'' local minimum points at which all the eigenvalues of the Hessian are zero, e.g., $x = 0$ for $f(x) = x^4$. 
Hereafter, with a slight abuse of terminology, we call $S$ as \emph{the set of local minima}.
The goal of ALOE is to efficiently classify all the points in $\mathcal{X}$ into either of $S$ or $\bar{S} := \mathcal{X} \setminus S$ with as small number of function evaluations as possible. 

We employ GP for modeling the unknown function $f$.
Specifically, we assume that the prior distribution of $f$ is $\mathcal{G}\mathcal{P}(0, k(\bm x, \bm x^\prime))$, where $k(\bm x, \bm x^\prime): D \times D \to \R$  is a PD kernel. 
Consider the $t^{\rm th}$ step where a sequence of the input points $\bx_1, \ldots, \bx_t$ on $D$ are selected by an AL method.
Then, the joint distribution $(f (\bm x_1), \ldots, f (\bm x_t))^\top$ follows the $t$-dimensional normal distribution $\mathcal{N}_t ({\bm\mu} _t, {\bm {K}} _ t)$ with the mean vector ${\bm\mu}_t=(0,\ldots,0)^\top \equiv {\bm{0}}_t$ and the covariance matrix $\bm K_t$ whose $(i,j)^{\rm th}$ element is $k(\bm x_i,\bm x_j)$. 
The output $ y_i $ is assumed to be obtained as $ y_i = f (\bm x_i) + \ve_i $,  where $ \ve_1, \ldots, \ve_t $ are independent random variables from $\mathcal{N} (0, \sigma^2)$.
Furthermore, the posterior distribution of $ f $ is also represented as a GP whose mean $ \mu_t (\bm x) $, variance $ \sigma^2_t (\bm x) $ and covariance $ k_t (\bm x, \bm x ') $  are given by
\begin{align*}
\mu_t (\bm x) &= {\bm{k}}_t (\bm x) ^\top{\bm{C}}_t^{-1} {\bm{y}}_t, \ 
\sigma^2_t (\bm x)= k_t (\bm x,\bm x), \\
k_t(\bm x,\bm x')&= k(\bm x,\bm x')- {\bm{k}}_t (\bm x) ^\top {\bm{C}}_t^{-1} {\bm{k}}_t (\bm x^\prime)
\end{align*}
where ${\bm{k}}_t (\bm x)= (k(\bm x_1,\bm x),\ldots,k(\bm x_t,\bm x))^\top$, 
${\bm{C}}_t = ({\bm{K}}_t +\sigma^2 {\bm{I}}_t)$, 
$\bm y_t =(y_1,\ldots,y_t)^\top$ and 
${\bm{I}}_t$ is a $t$-dimensional identity  matrix.

\paragraph{GP derivatives}
We assume that the kernel function $k(\bm x, \bm x^\prime)$ is differentiable up to order four.
Many commonly used kernels including Gaussian and Linear kernels satisfy this assumption.
Under this assumption, it is known that the first and second derivatives of $\mathcal{G}\mathcal{P}(0,k(\bm x,\bm x'))$ is also GPs (e.g., \cite{Rasmussen:2005:GPM:1162254}, \cite{papoulis2002probability}).
Here, let $f^{(1)}_i$ and $f^{(2)}_{jk}$ be the  first and the second  derivatives of $f$ in the $i^{\rm th}$ and $(j,k)^{\rm th}$ elements, respectively.
Then, given the observations $(\bm x_1,y_1),\ldots,(\bm x_t,y_t)$, the posterior distribution of $f^{(1)}_i$ is also GP, and its mean, variance and covariance are respectively given by
\begin{align*}
 &\mu^{(1)}_{t,i} (\bm x) = {\bm{k}}^{(1)}_{t,i} (\bm x) ^\top {\bm{C}}_t^{-1} {\bm{y}}_t, ~ \{\sigma^{(1)}_{t,i} (\bm x) \}^2 = v^{(1)}_{t,i} (\bm x,\bm x), \\
 &v^{(1)}_{t,i} (\bm x,\bm x') = v^{(1)}_i (\bm x,\bm x') -  {\bm{k}}^{(1)}_{t,i} (\bm x) ^\top {\bm{C}}_t^{-1}  {\bm{k}}^{(1)}_{t,i} (\bm x'),
\end{align*}
%
where the $l^{\rm th}$ element of ${\bm{k}}^{(1)}_{t,i} (\bm x)$ is $\partial k(\bm x_l,\bm x) /\partial x_i$ and $v^{(1)}_i (\bm x,\bm x')=\partial^2 k(\bm x,\bm x')/\partial x_i \partial x'_i .$
Similarly, the posterior distribution of  $f^{(2)}_{jk}$ is also GP, and its 
mean,
variance,
and covariance
are respectively given by
\begin{align*}
\hspace*{-2.5mm} & \mu^{(2)}_{t,jk} (\bm x) = {\bm{k}}^{(2)}_{t,jk} (\bm x) ^\top {\bm{C}}_t^{-1} {\bm{y}}_t, ~ \{\sigma^{(2)}_{t,jk} (\bm x)\}^2 = v^{(2)}_{t,jk} (\bm x,\bm x), \\
\hspace*{-2.5mm} & v^{(2)}_{t,jk} (\bm x,\bm x') = v^{(2)}_{jk} (\bm x,\bm x') -  {\bm{k}}^{(2)}_{t,jk} (\bm x) ^\top {\bm{C}}_t^{-1}  {\bm{k}}^{(2)}_{t,jk} (\bm x'),
\end{align*}
where the $l$th element of the second derivative ${\bm{k}}^{(2)}_{t,jk} (\bm x)$ is given by $\partial^2 k(\bm x_l,\bm x) /\partial x_j \partial x_k$, and
\begin{align*}
 v^{(2)}_{jk} (\bm x,\bm x')=\partial^4 k(\bm x,\bm x')/\partial x_j \partial x_k \partial x'_j \partial x'_k.
\end{align*}

\section{Proposed method}\label{sec3}
In this section, we describe the proposed ALOE method 
 for  efficiently identifying the set of local minima $ S$   in   \eqref{eq:localMinima}.
At the step, ALOE  estimates whether each $\bx \in \mathcal{X} $ is included in $ S$ using  the CIs of the gradients  and       the  Hessian   minimum eigenvalue.  
Figure \ref{fig:ALOE} illustrates the behavior of ALOE.

\begin{figure*}[t]
\begin{center}
\scalebox{1}{
 \begin{tabular}{cccc}
 \includegraphics[width=0.225\textwidth]{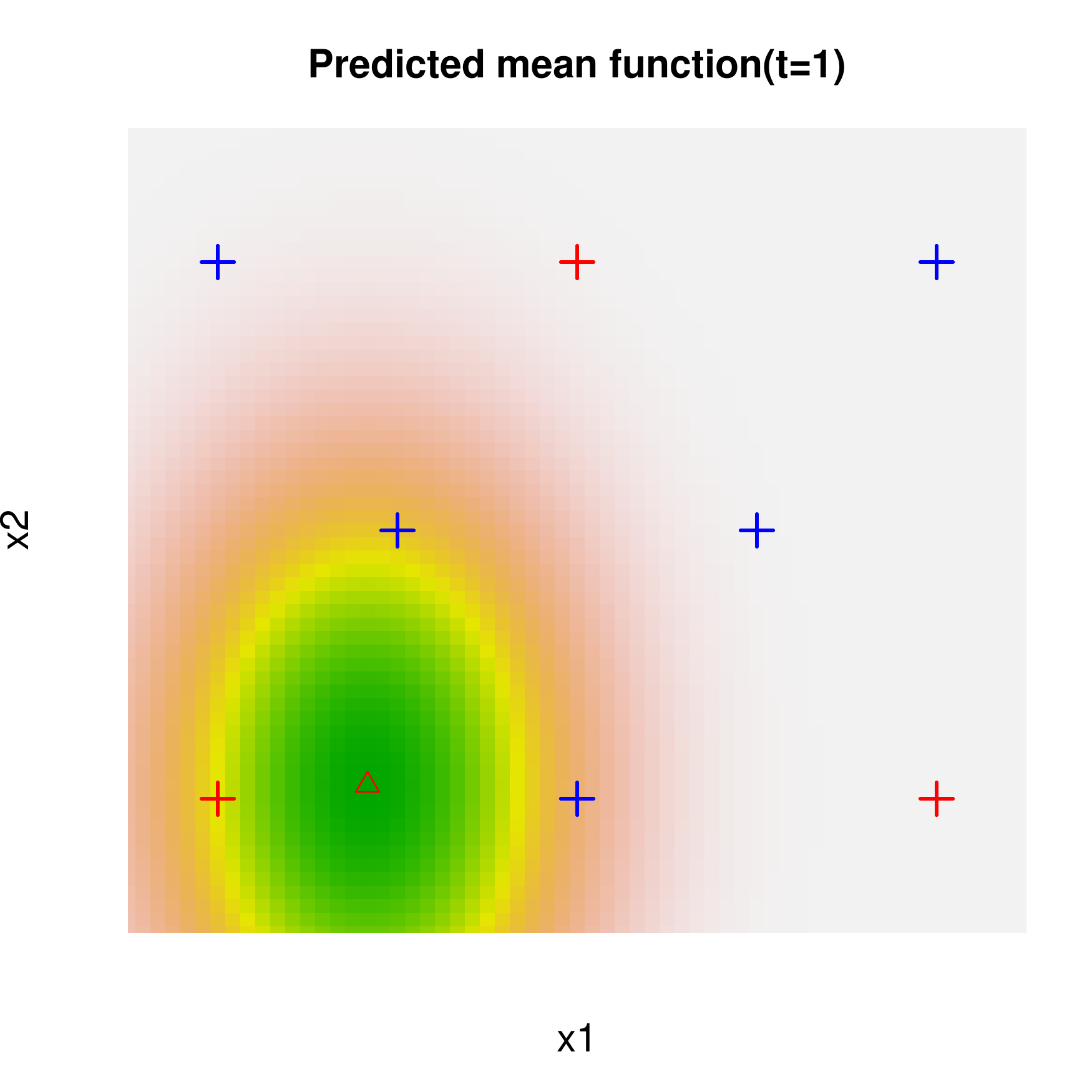} & 
 \includegraphics[width=0.225\textwidth]{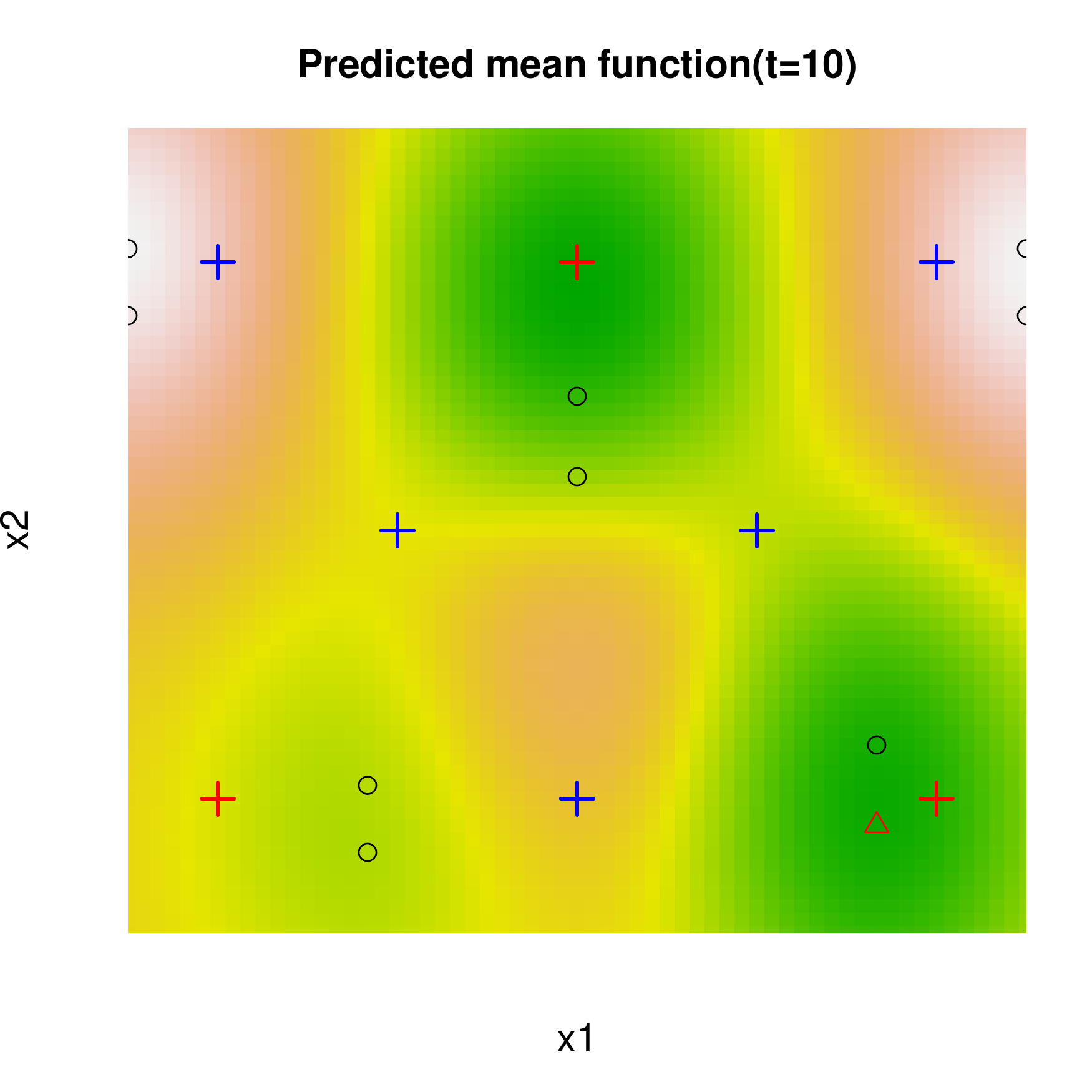} & 
 \includegraphics[width=0.225\textwidth]{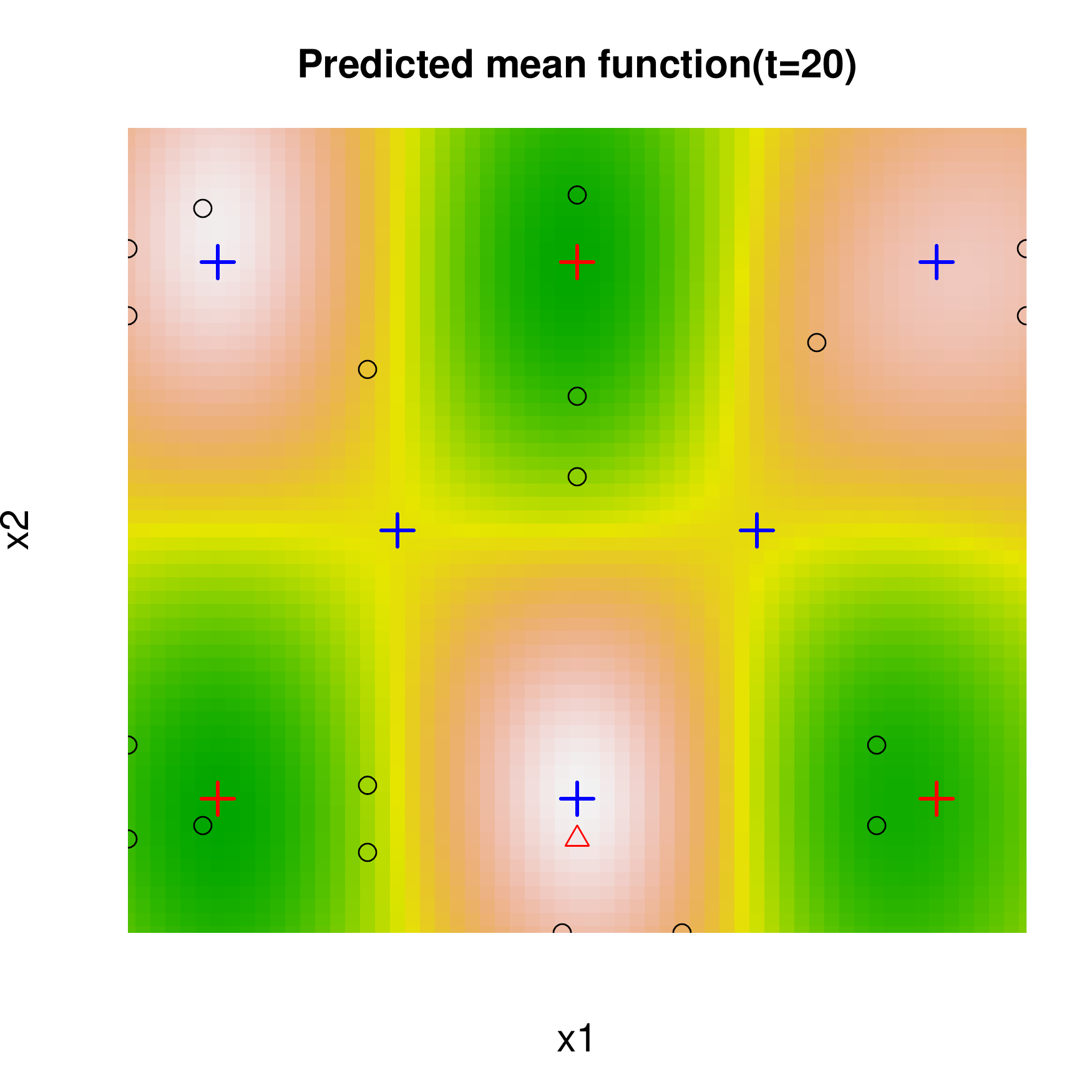} & 
 \includegraphics[width=0.225\textwidth]{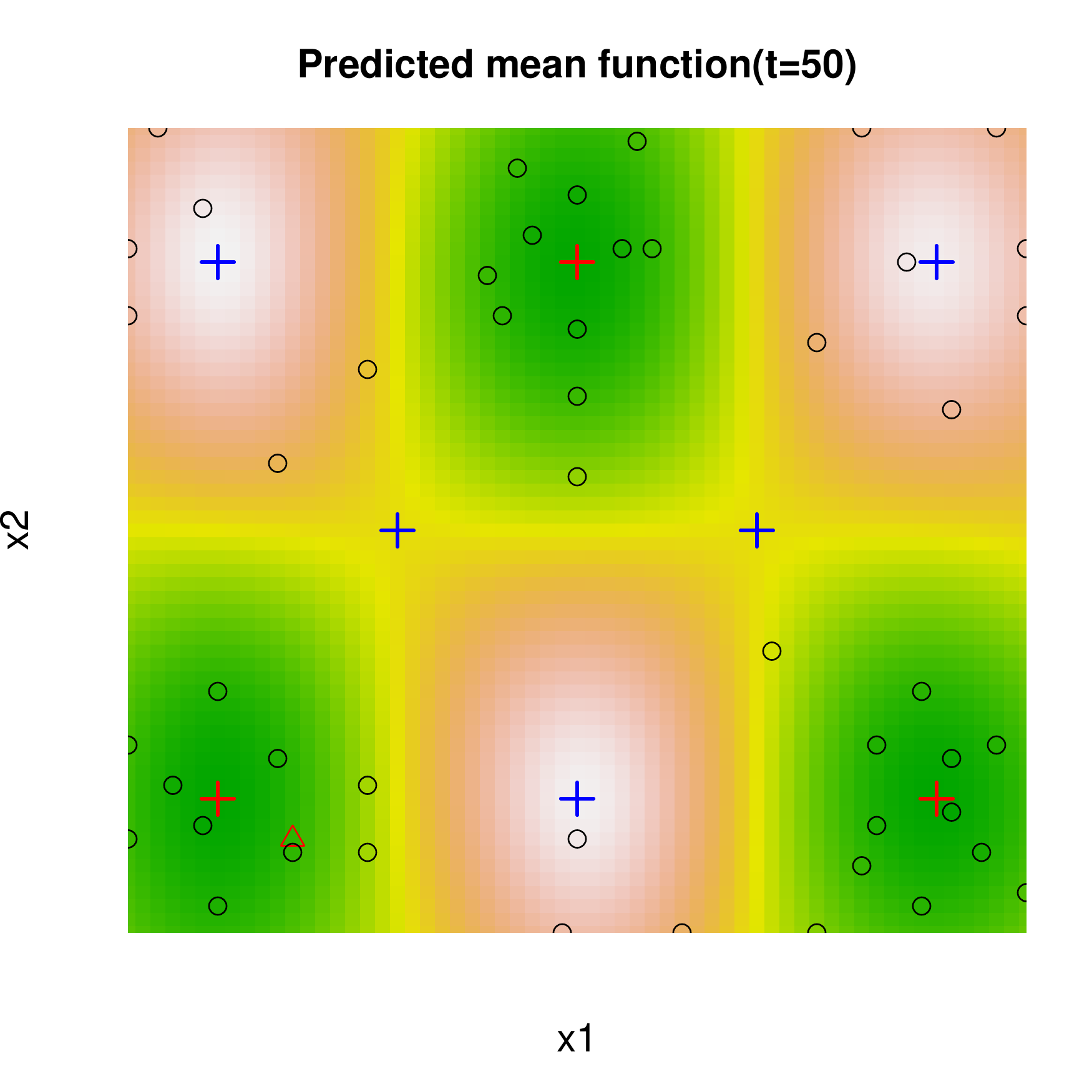} \\ 
 \includegraphics[width=0.225\textwidth]{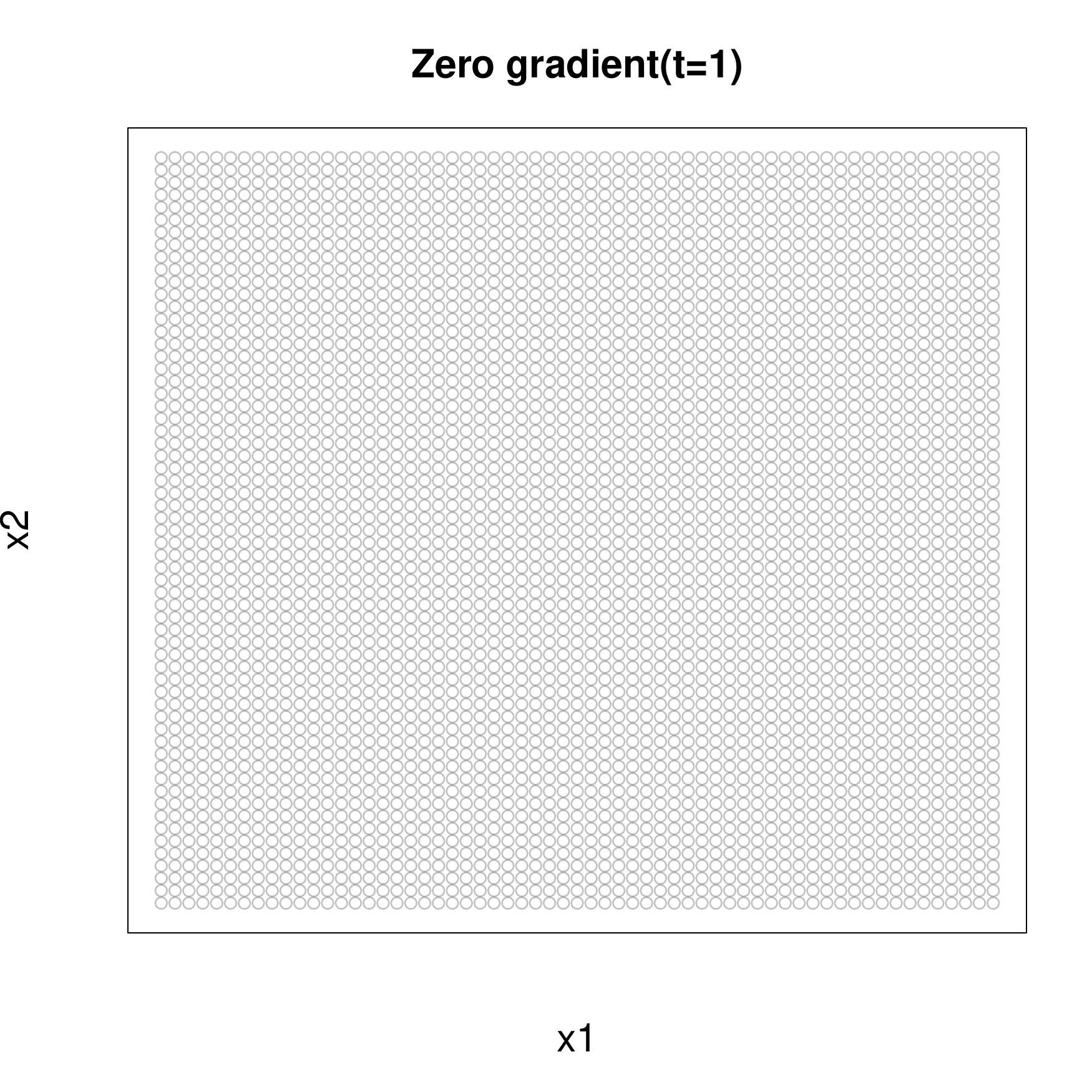} & 
 \includegraphics[width=0.225\textwidth]{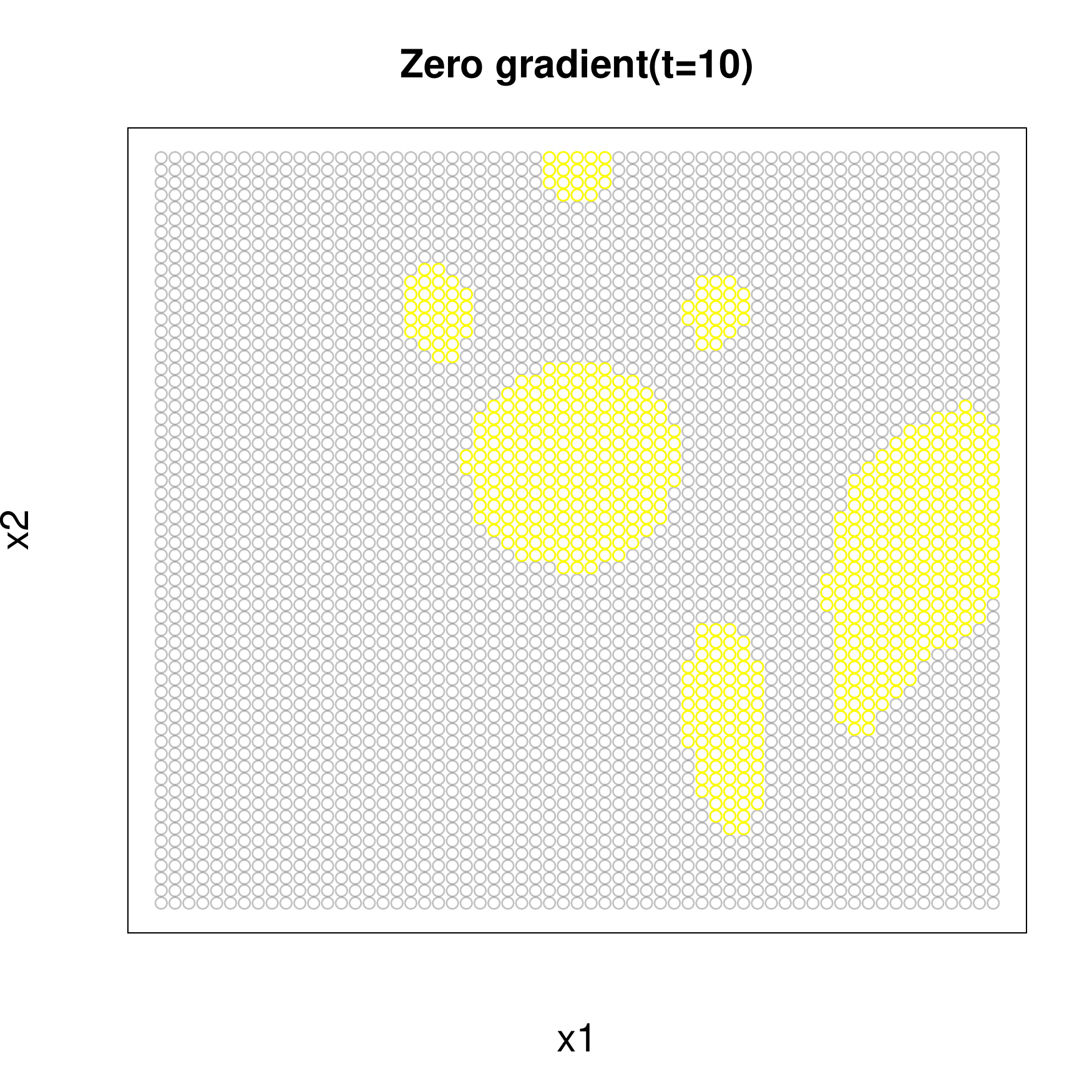} & 
 \includegraphics[width=0.225\textwidth]{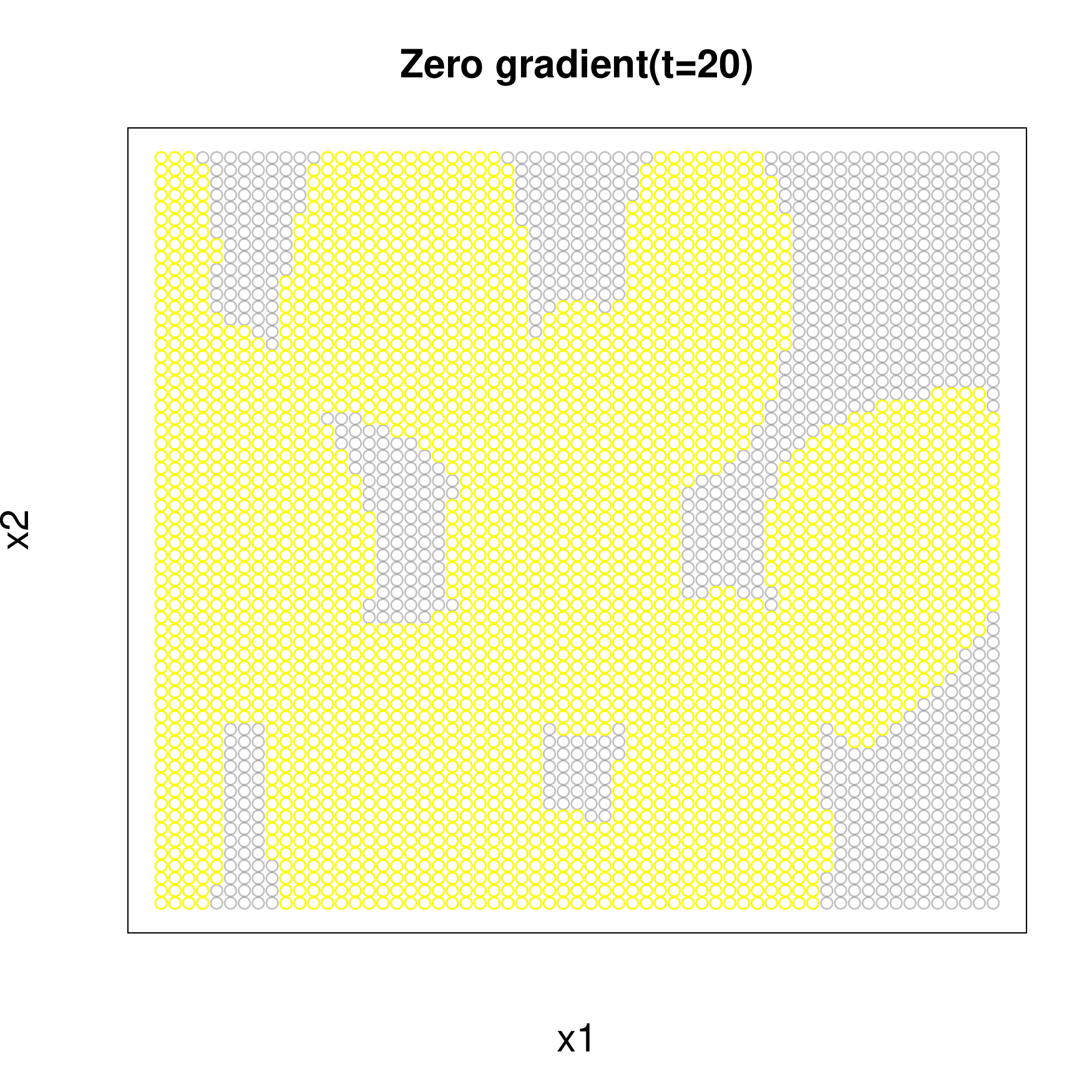} & 
 \includegraphics[width=0.225\textwidth]{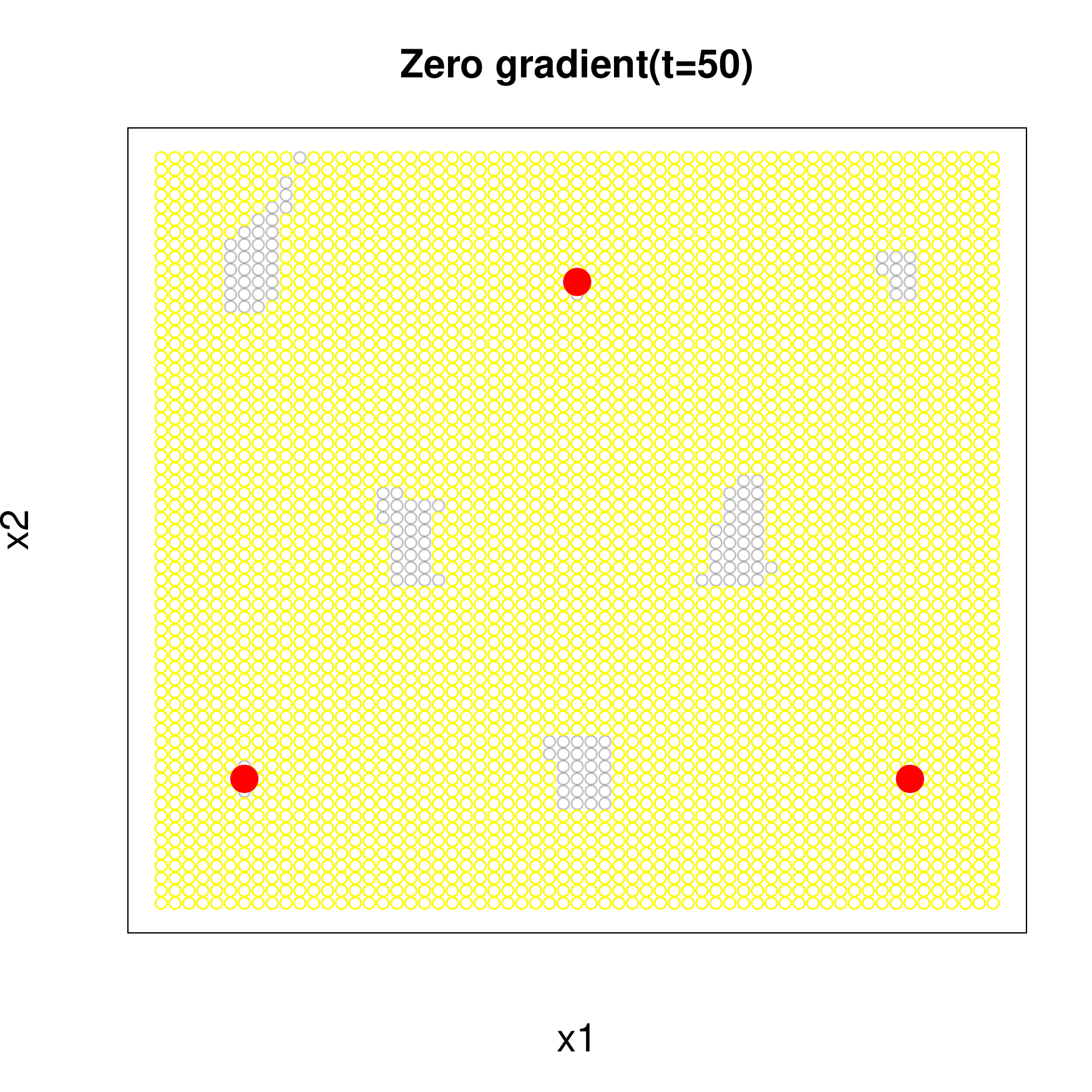} \\ 
 \includegraphics[width=0.225\textwidth]{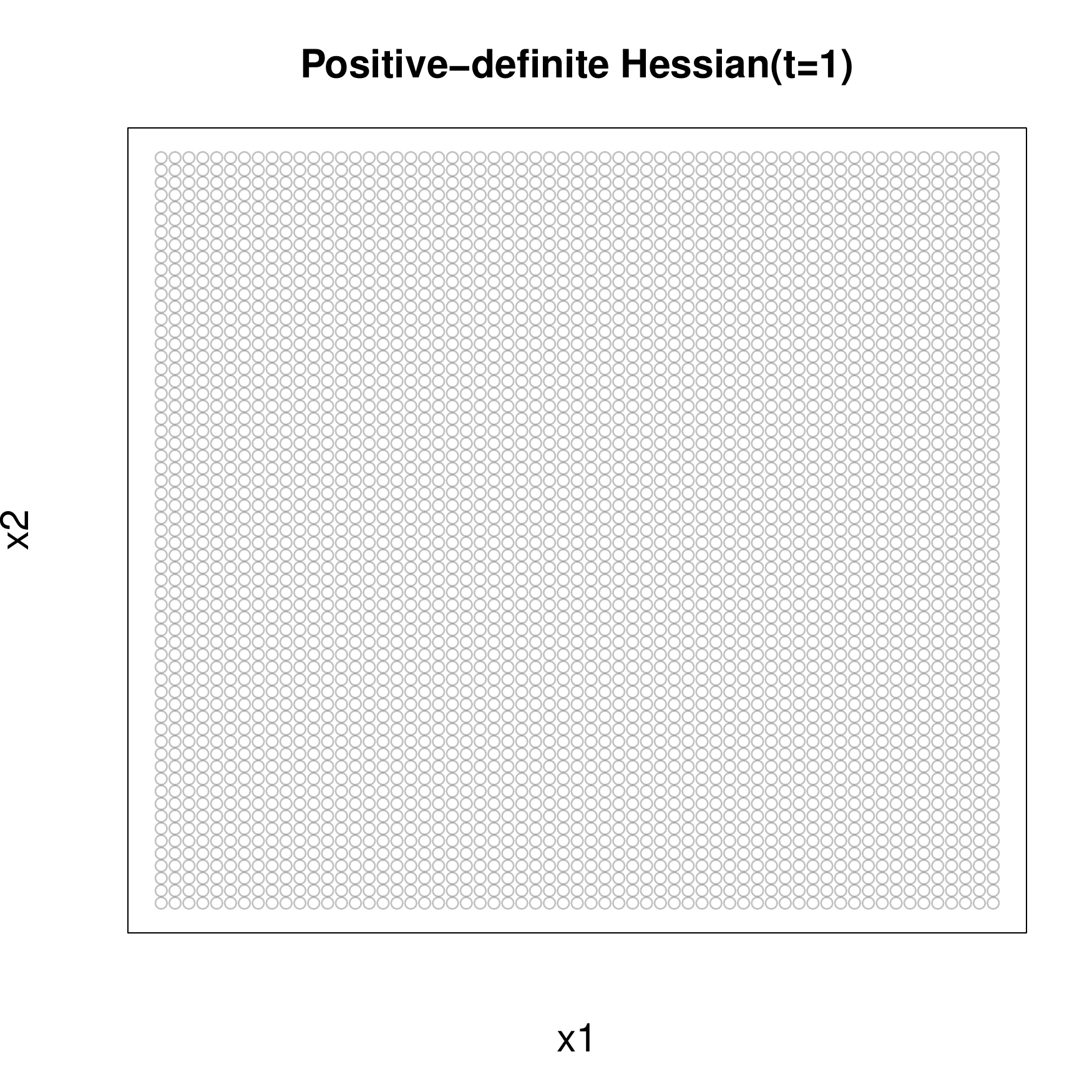} & 
 \includegraphics[width=0.225\textwidth]{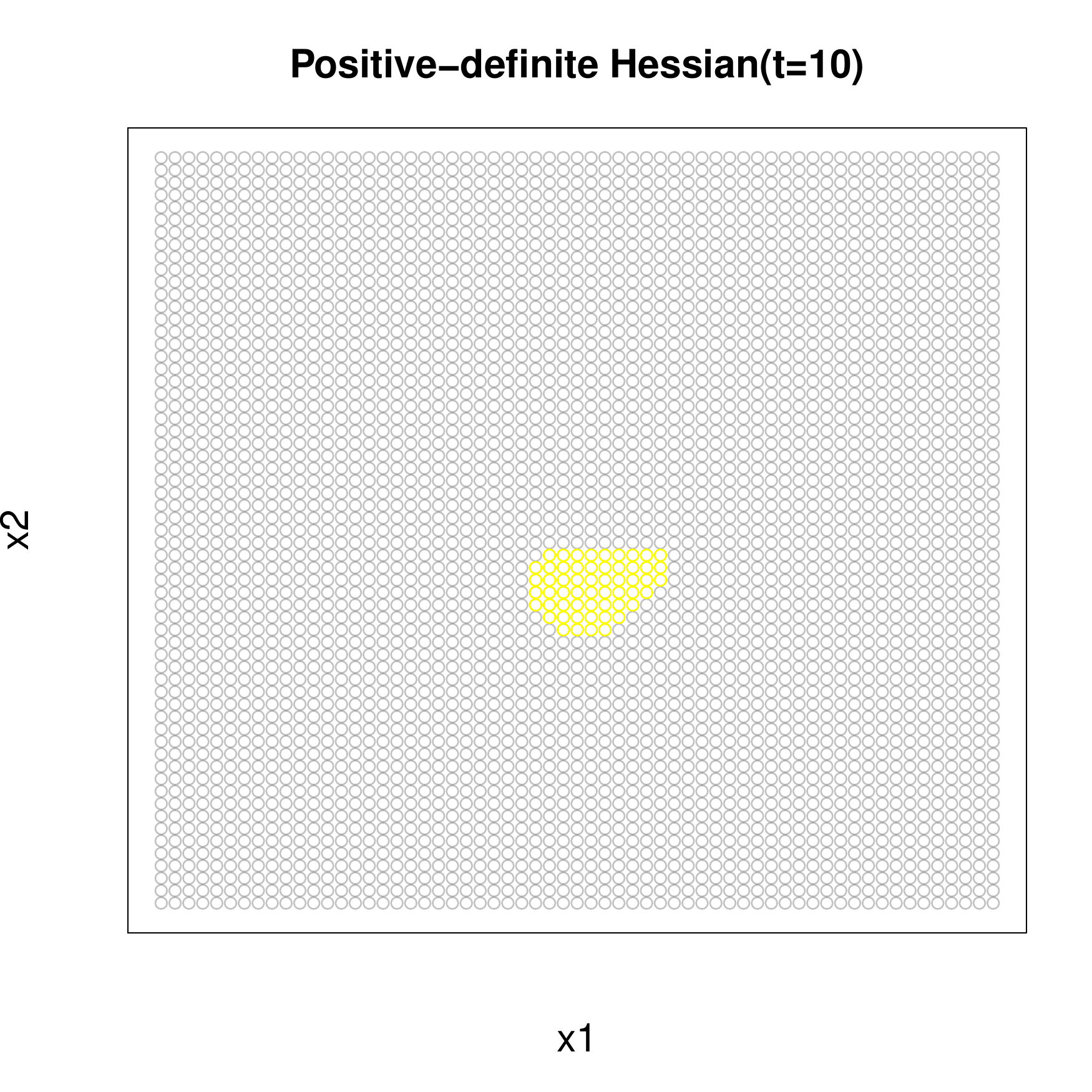} & 
 \includegraphics[width=0.225\textwidth]{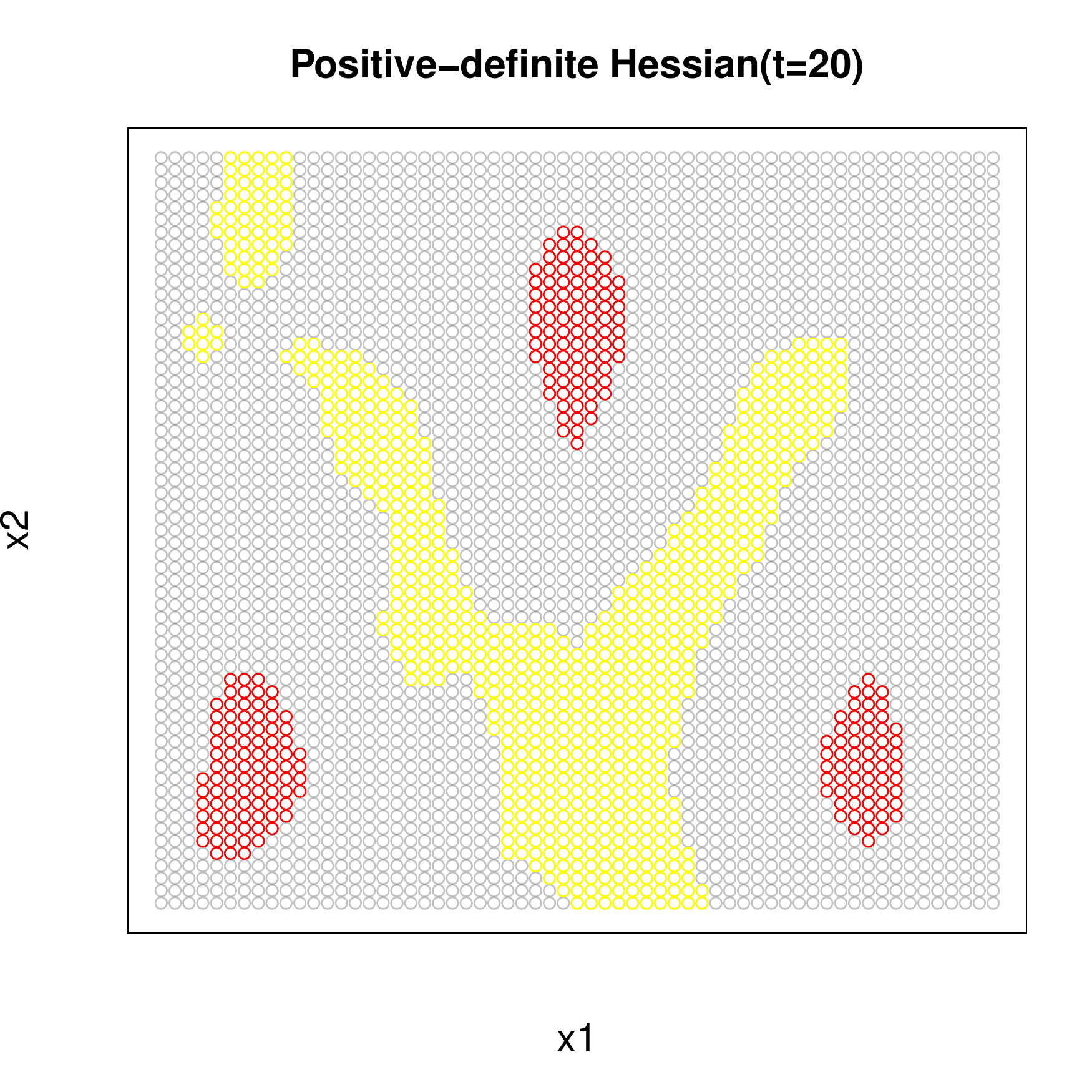} & 
 \includegraphics[width=0.225\textwidth]{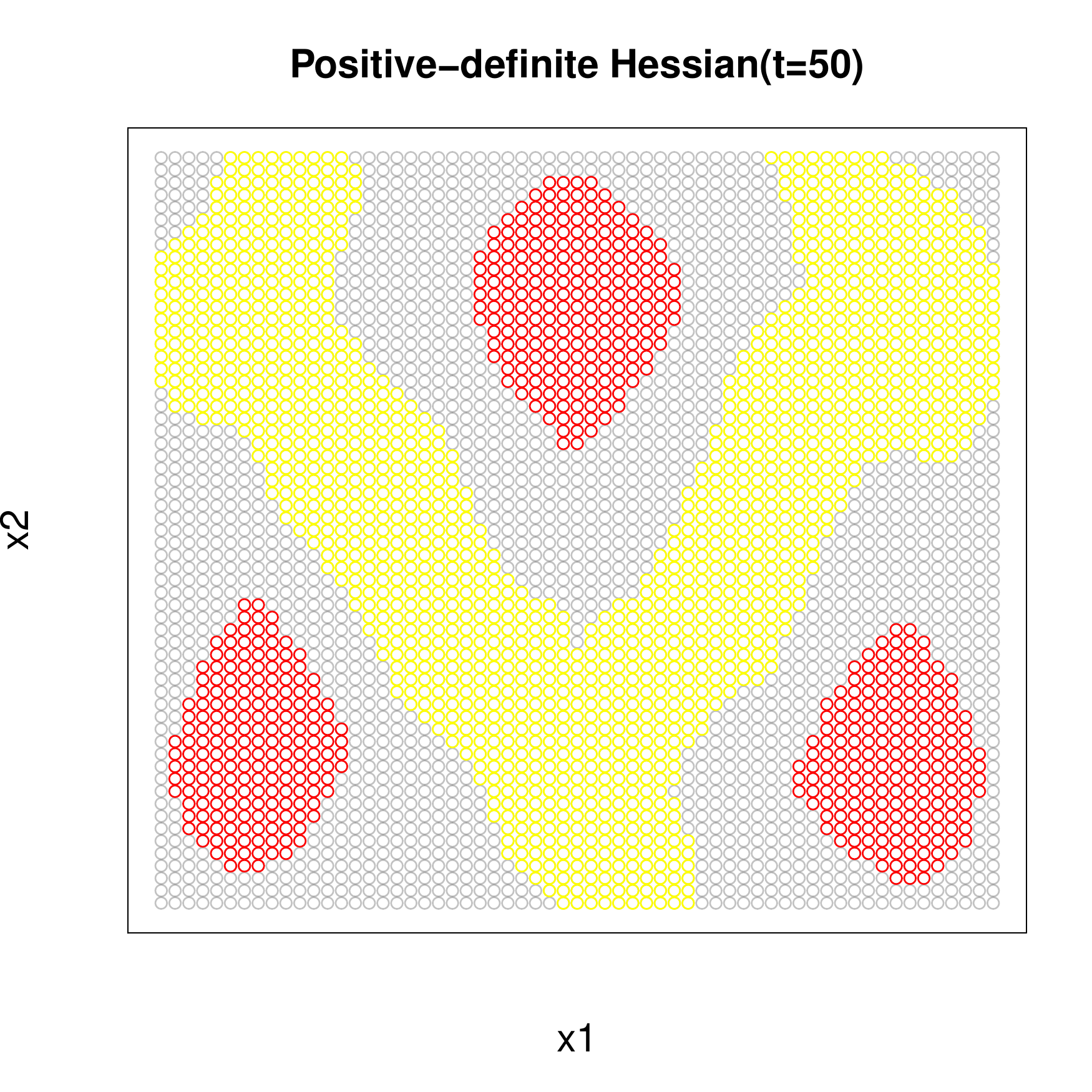}
 \end{tabular}}
\end{center}
 \caption{An example of ALOE in a synthetic 2-dimensional data.
The top, middle and bottom row show the predicted mean functions, zero-gradient regions, and PD Hessian regions, respectively (in the middle and bottom rows, red, yellow, and gray colors indicate zero-gradient/PD-Hessian regions, NON zero-gradient/PD-Hessian regions, and unknown regions, respectively).
In the top row, the red and blue crosses indicate local minimum and local maximum points, respectively. 
ALOE effectively selected the sequence of input points from the region close to the local minima, and efficiently identified all the local minimum points at $t=50$.
}
 \label{fig:ALOE}
\end{figure*}

\subsection{Local minimum estimation based on the CIs of GP derivatives}\label{GPiden}
For each  $\bx \in \mathcal{X}$, we define the CIs of   $f^{(1)}_i (\bx)$ at the $t^{\rm th}$ iteration as $ Q^{(1)}_{t,i} (\bx)  =[ l^{(1)}_{t,i} (\bx),u^{(1)}_{t,i} (\bx)]$, 
where $ l^{(1)}_{t,i} (\bx) = \mu^{(1)}_{t,i} (\bx) -\beta^{1/2}_{t} \sigma^{(1)}_{t,i} (\bx)$, 
$ u^{(1)}_{t,i} (\bx) = \mu^{(1)}_{t,i} (\bx) + \beta^{1/2}_{t} \sigma^{(1)}_{t,i} (\bx)$ 
and $\beta^{1/2}_{t} \geq0$. 
Then, 
 by using an accuracy parameter $\epsilon^{(1)}_i >0$ we define 
 $G^{(1)}_{t,i}$ and $\bar{G}^{(1)}_{t,i}$ as 
\begin{align}
\hspace*{-5mm}  G^{(1)}_{t,i} &= \{ \bx \in \mathcal{X} \ | \ - \epsilon^{(1)} _i < l^{(1)}_{t,i} (\bx) \land  \ u^{(1)}_{t,i} (\bx) < \epsilon^{(1)} _i \}, \label{eq:L1i}  \\
\hspace*{-5mm} \bar{G}^{(1)}_{t,i} &= \{ \bx \in \mathcal{X} \ | \ 0 \leq l^{(1)}_{t,i} (\bx) \lor \ u^{(1)}_{t,i} (\bx) \leq 0 \}.  \label{eq:nL1i} 
\end{align}
Here, $G^{(1)}_{t,i} $ is the set of points at which  the CIs  of the gradients fall within $ [- \epsilon^{(1)}_i, \epsilon^ {(1)}_i ] $,  
i.e., the set of points expected to have zero gradient with the accuracy $ \epsilon ^{(1)} _i$.
Similarly, $\bar{G}^{(1)}_{t,i}$ is the set of points at which the CIs of the gradients are sufficiently away from zero.

Next, we consider the identification of the points at which  the Hessian of $ f $ is  PD.
Note that  the minimum eigenvalue of a matrix is positive  if and only if the Hessian is PD. 
Using this equivalence, we perform   identification of PD Hessian using the CI of the minimum eigenvalue.
For each point $\bx \in \mathcal{X}$, we define 
 the CI of the minimum eigenvalue $\lambda (\bx)$  as 
$Q^{(2)}_{t} (\bx)  =  [ l^{(2)}_{t},u^{(2)}_{t}] $, where 
$ l^{(2)}_{t} = \lambda_{t} (\bx) - \gamma^{1/2}_{t} \varsigma^{(2)}_t (\bx)$, 
$ u^{(2)}_{t} = \lambda_{t} (\bx) + \gamma^{1/2}_{t} \varsigma^{(2)}_t (\bx)$, 
$\gamma^{1/2}_{t} \geq 0$ and $\lambda_{t}(\bx)$ is the minimum eigenvalue of the $d \times d$ matrix whose $(j,k)^{\rm th}$ element is  
$\mu^{(2)}_{t,jk} (\bx)$. 
On the other hand, since the variance of $\lambda (\bx)$ is not readily available, we use  
 $\varsigma^{(2)}_t (\bx)$ defined as 
$
\varsigma^{(2)}_t (\bx) =   \max_{j,k \in [d]} \sigma^{(2)}_{t,jk} (\bx) 
$, where $[d]:=\{1,\ldots,d \}$. 
As shown  in section \ref{sec4}, 
by appropriately adjusting $\gamma_t$, 
$\lambda(\bx)$ is shown to be included in $Q^{(2)}_t (\bx) $ with   high probability. 
 Then, using an accuracy parameter $\epsilon^{(2)} >0$ we define $H^{(2)}_{t}$ and 
$\bar{H}^{(2)}_{t}$ as 
\begin{align}
H^{(2)}_{t} &= \{ \bx \in \mathcal{X} \ | \ l^{(2)} _{t} > -\epsilon ^{(2)} \} , \label{eq:minsei} \\
\bar{H}^{(2)}_{t} &= \{ \bx \in \mathcal{X} \ | \ u^{(2)} _{t} < \epsilon ^{(2)} \}, \label{eq:minhu}
\end{align}
where $H^{(2)}_{t} $ (resp. $\bar{H}^{(2)}_{t} $) is the set of points where the minimum eigenvalue is expected to be positive (resp. negative)  with an accuracy $ \epsilon^{(2)}$.  
Then, from \eqref{eq:L1i}--\eqref{eq:nL1i} and  \eqref{eq:minsei}--\eqref{eq:minhu}, we estimate $S$ as follows:
\begin{definition}[$S$ estimation]
The estimates of $S$ and $\bar{S} := \mathcal{X} \setminus S $ are respectively defined as 
\begin{equation}
  \widehat{S}_t= H^{(2)}_{t} \cap  \bigcap_{i=1}^d  G^{(1)}_{t,i}, \
\widehat{\bar{S}}_t= \bar{H}^{(2)}_{t} \cup   \bigcup_{i=1}^d \bar{G}^{(1)}_{t,i}. \nonumber
\end{equation}
The set of remaining points at step $t$ is  defined as $U_t = \mathcal{X} \setminus (   \widehat{S}_t \cup \widehat{\bar{S}}_t)$. Figure \ref{zu1} shows an example of CIs.
\end{definition}

\begin{figure*}[t]
\begin{center}
\scalebox{1}{
\includegraphics[width=0.49\textwidth]{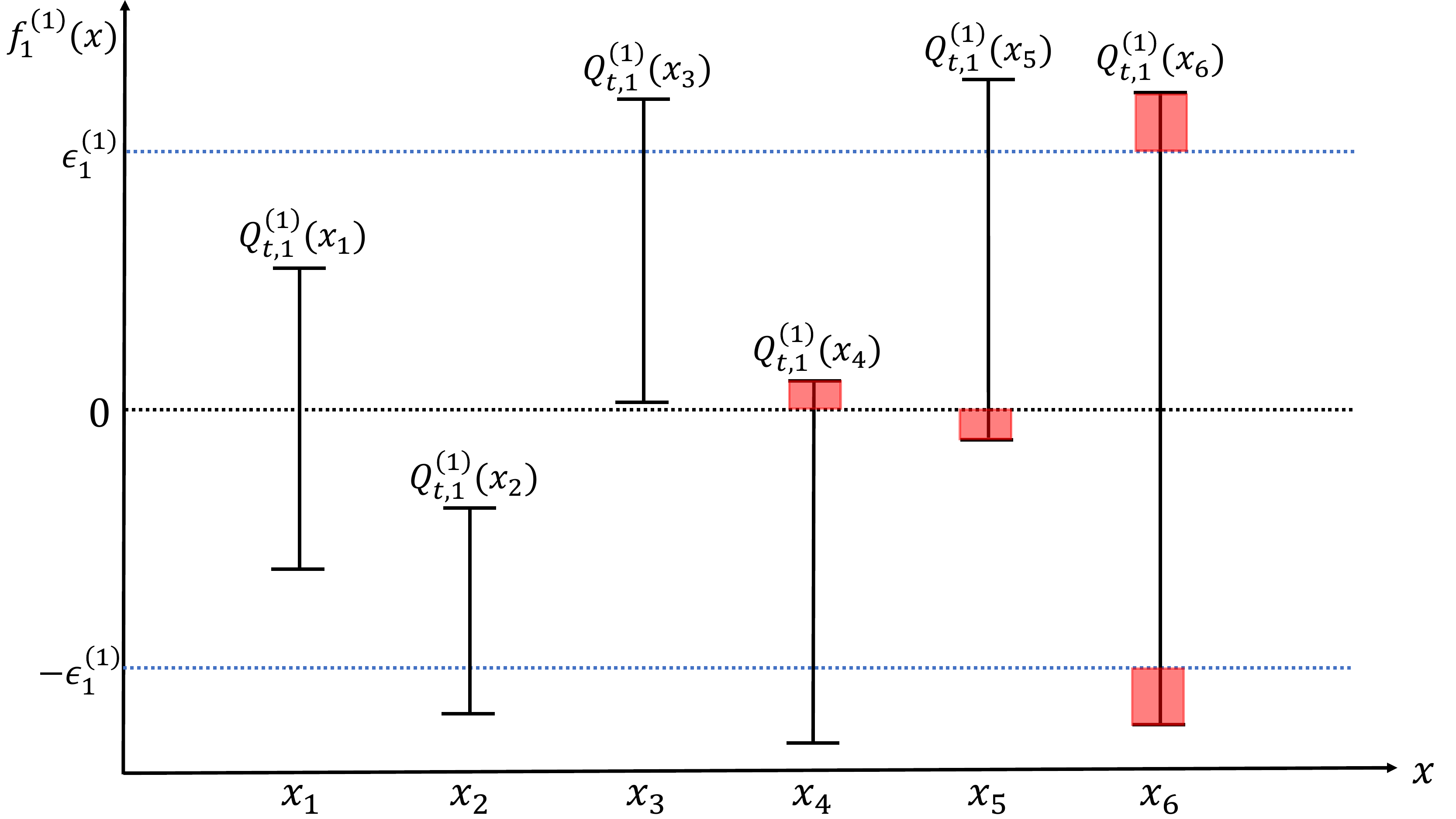}
\includegraphics[width=0.49\textwidth]{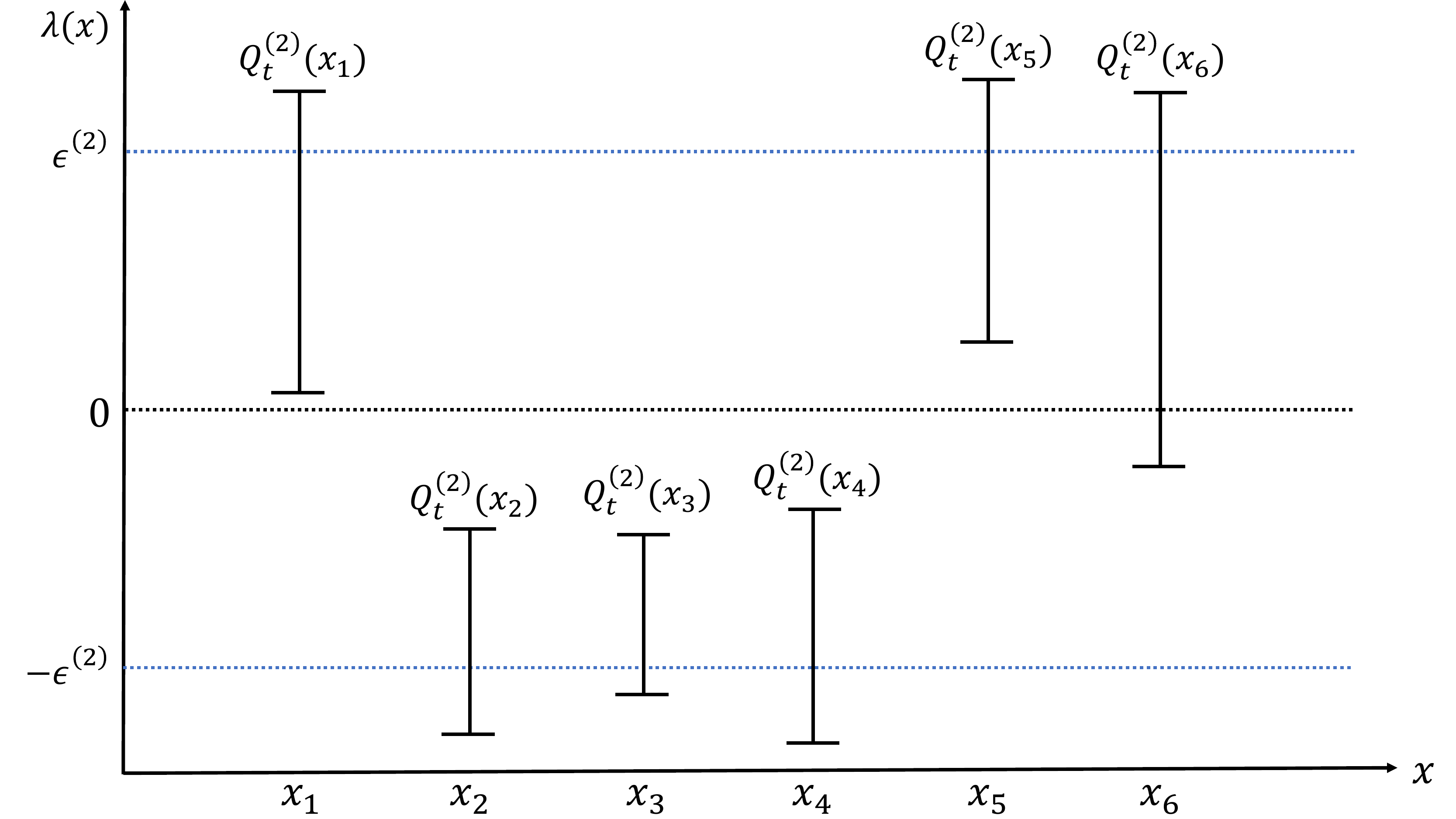}
}
\end{center}
\caption{A schematic diagram of the CIs and the violation in ALOE. Left and right plots show the CIs of the gradients and the Hessian minimum eigenvalues of one-dimensional six input points $x_1, \ldots, x_6$.
The left plot indicates that the gradients at $x_1$, $x_4$, $x_5$, and $x_6$ could be zero, while the right plot indicates that the Hessians at $x_1$, $x_5$, and $x_6$ could be PD. 
Here, $S_t := \{x_1\}, \bar{S}_t = \{x_2, x_3, x_4\}, U_t = \{x_5, x_6\}$.
The red parts indicate the violations which play an important role in the AF of ALOE.
}
\label{zu1}
\end{figure*}

\subsection{Acquisition function by predicted violations}\label{AF}
Based on  $\widehat{S}_t$ and $\widehat{\bar{S}}_t$, 
we propose an AF for efficiently enumerating local minima. 
The proposed AF $a_t (\bx)$ consists of components as 
\begin{align}
a_t (\bx) = r_t \sigma^2_t (\bx) + (1-r_t) b_t (\bx) , \label{eq:AF}
\end{align}
where $r_t = \{ 0,1 \}$ adjusts the trade-off between two components. 
The first component $\sigma^2_t (\bx)$ is merely the posterior variance of $f(\bx)$. 
Thus, when $r_t =1$, the AF is reduced to the AF of \emph{Uncertainty sampling} \cite{settles2009active}. 
The second component $b_t (\bx)$ is a specific function designed for reducing the uncertainties in the 
gradients and Hessian minimum eigenvalue for the task of enumerating local minima. 
In the remainder of this section, we describe the detail of $b_t (\bx)$. 

First, we define  \emph{violations} (see, Figure \ref{zu1}).   
Remembering that the goal is to classify each point $\bx \in \mathcal{X}$ into either of $S$ or $\bar{S}$, 
the violation of the CI of $f^{(1)}_i $ at $\bx$ is defined as 
\begin{align}
V^{(1)}_{t,i} (\bx) 
=   \min \{   \xi ( u^{(1)}_{t,i} (\bx) ) , \xi ( -l^{(1)}_{t,i} (\bx) ), 
   \xi (u^{(1)}_{t,i,\epsilon^{(1)}_i}(\bx))+
\xi (l^{(1)}_{t,i,\epsilon^{(1)}_i} (\bx)) \}, \nonumber 
\end{align}
where  $u^{(1)}_{t,i,\epsilon^{(1)}_i} (\bx)= u^{(1)}_{t,i} (\bx) -\epsilon^{(1)}_i$ and $l^{(1)}_{t,i,\epsilon^{(1)}_i} (\bx)= -\epsilon^{(1)}_i-l^{(1)}_{t,i} (\bx)$. 
Here,  $\xi(a)=a$ if      $a>0$ and otherwise $0$.
The $b_t (\bx)$ in the second component of the AF is designed to be able to select the next point 
such that the largest violation is maximally reduced.

Let $\bx^+ \in U_t $ be the input at which the sum of the violation $\sum^d_{i=1} V^{(1)}_{t,i} (\bx^+ ) $ 
is largest, i.e., 
$$
\bx ^+ _t := \argmax _{ \bx \in U_t } \sum_{i=1}^d V^{(1)}_{t,i} (\bx) .  
$$
Unfortunately, since the gradient $f^{(1)}_i (\bx^+ )$ cannot be directly observed, it is not sufficient to 
simply select the input $\bx ^+$ as the next  input point. 
Indeed, we need to select the input point 
$\bx^* \in D $ 
such that 
the \emph{predicted violation} of 
$f_i^{(1)}(\bx^+)$ 
can be maximally reduced 
by evaluating $f(\bx^*)$.
%
%
Let $\sigma_{t, i}^{(1)}(\bx; \bx^*)$ 
be the posterior variance of 
$f_{i}^{(1)}(\bx)$
when
the function $f(\bx^*)$ is newly evaluated. 
By replacing
$\sigma_{t, i}^{(1)}(\bx)$ in $Q^{(1)}_{t,i} (\bx)$
to 
$\sigma_{t, i}^{(1)}(\bx; \bx^*)$, 
we obtain
the the predicted CIs
$[l_{t, i}^{(1)}(\bx^+; \bx^*), u_{t, i}^{(1)}(\bx^+; \bx^*)]$ 
~\footnote{
Here, the mean $\mu_{t, i}^{(1)}(\bx)$ is not replaced since its update is unknown before we actually evaluate $f(\bx^*)$.
}.
Then, 
by replacing 
$l_{t, i, \epsilon_i}^{(1)}(\bx)$ and $u_{t, i, \epsilon_i}^{(1)}(\bx)$
to 
similarly defined 
$l_{t, i, \epsilon_i}^{(1)}(\bx; \bx^*)$ and $u_{t, i, \epsilon_i}^{(1)}(\bx; \bx^*)$,
respectively, 
we similarly define the predicted violation
$V_{t, i}^{(1)}(\bx^+; \bx^*)$, 
which represents the predicted violation 
when
$f(\bx^*)$ is newly evaluated. 
In summary,
the second component of the AF is defined as
\begin{align*}
 b_t(\bx) = \sum_{i=1}^d (V_{t,i}^{(1)}(\bx_t^+) - V_{t,i}^{(1)}(\bx_t^+; \bx)). 
\end{align*}
Algorithm \ref{ALG1}. shows the flow of ALOE.
\begin{algorithm}[t]                  
\caption{Local minima identification using GP derivatives}         
\label{ALG1}                          
\begin{algorithmic}[1]                  
\REQUIRE 
 Initial training data,  
GP prior $\mathcal{G}\mathcal{P} (0,k(\bx,\bx'))$
\ENSURE Estimated sets $\widehat{S}$ and $ \widehat{\bar{S}}$
\STATE $\widehat{S}_0 \gets \emptyset$, $\widehat{\bar{S}}_0\gets \emptyset$, $U_0 \gets \mathcal{X}$
\STATE $t \gets 1$
\WHILE{$U_{t-1} \neq \emptyset$}
	\STATE $\widehat{S}_t \gets \widehat{S}_{t-1}$, $\widehat{\bar{S}}_t\gets \widehat{\bar{S}}_{t-1}$, $U_t \gets U_{t-1}$
	\FORALL{$ \bx \in \mathcal{X}$}
	\STATE Compute confidence intervals $Q^{(1)}_{t,i} (\bx)$ and  $Q^{(2)}_t (\bx)$ from GP derivatives
	\ENDFOR
	\STATE 
Compute $G^{(1)}_{t,i} $, $\bar{G}^{(1)}_{t,i}$, $H^{(2)}_{t}$ and  $\bar{H}^{(2)}_{t}$ for each\ $i \in [d]$ 
	\FORALL{$ \bx \in \mathcal{X}$}
		\IF{$\bx \in H^{(2)}_{t} \cap \bigcap_{i=1}^d G^{(1)}_{t,i}  $} 
		\STATE $\widehat{S}_t \gets \widehat{S}_{t-1} \cup \{ \bx \}$, $U_t \gets U_{t-1} \setminus \{ \bx \}$
		 \ELSIF{$\bx \in \bar{H}^{(2)}_{t} \cup \bigcup_{i=1}^d \bar{G}^{(1)}_{t,i}  $} 
		\STATE $\widehat{\bar{S}}_t \gets \widehat{\bar{S}}_{t-1} \cup \{ \bx \}$, $U_t \gets U_{t-1} \setminus \{ \bx \}$
		 \ENDIF
	\ENDFOR
	\STATE Compute   $\bx_{t}$ from \eqref{eq:AF} 
\STATE $y_t \gets f({\bm{x}}_t ) +\varepsilon_t$
\STATE $t \gets t+1$
\ENDWHILE
\STATE $\widehat{S} \gets \widehat{S}_{t-1}$, $\widehat{\bar{S}} \gets \widehat{\bar{S}}_{t-1}$
\end{algorithmic}
\end{algorithm}

\section{Theoretical results}\label{sec4}
We provide   theorems on the performance and convergence of Algorithm \ref {ALG1}.
First, the following theorem holds:
\begin{theorem}\label{thm:seido}
Let   $\epsilon^{(1)}_1,\ldots, \epsilon^{(1)}_d , \epsilon^{(2)}$ be positive numbers, and let $\epsilon = \min\{ \epsilon^{(1)}_1,\ldots, \epsilon^{(1)}_d , \epsilon^{(2)} \}$. 
For any 
$\delta \in (0,1)$, define $\beta_t = 2\log ( (d+1) |\mathcal{X}| \pi^2 t^2/(6 \delta) ) $
, $\gamma_{t} = 2d^2 \log (d^2 (d+1) |\mathcal{X}| \pi^2 t^2 /(6 \delta ))$ and
\begin{align}
\eta_t = \max \{ \max_{\bx \in \mathcal{X}} 2\beta^{1/2}_{t} \sigma^{(1)}_{t,1} (\bx), \ldots , \max_{ \bx \in \mathcal{X}} 2\beta^{1/2}_{t} \sigma^{(1)}_{t,d} (\bx), 
\max_{\bx \in \mathcal{X}} \gamma^{1/2}_{t} \varsigma^{(2)}_t (\bx)  \}.\nonumber 
\end{align}
Then, 
Algorithm \ref{ALG1} completes classification after at least the minimum positive integer $T$ trials that satisfy the following inequality 
$\eta^2_T \leq\epsilon^2$. 
Moreover, with probability at least 
 $1-\delta$, for any $t \geq 1$, $\bx \in \mathcal{X}$ and $i \in [d]$ it holds that 
$$
\bx \in \widehat{S}_t \Rightarrow -\epsilon^{(1)}_i < f^{(1)}_i (\bx) < \epsilon^{(1)}_i \land \ \lambda (\bx) > -\epsilon^{(2)} 
$$ and 
$$
\bx \in \widehat{\bar{S}}_t \Rightarrow   f^{(1)}_i (\bx) \neq 0 \lor \ \lambda (\bx) < \epsilon^{(2)} .
$$
\end{theorem}
\begin{proof}
First, when the inequality on $\eta_T $ holds, 
the lengths of  $Q^{(1)}_{T,i} (\bx)$ and $Q^{(2)}_{T} (\bx)$ are less than 
$\epsilon^{(1)}_i$ and $2\epsilon^{(2)}$, respectively. 
Hence, from  classification rules,  all the points are classified. 
Next, noting that the posterior of $f^{(1)}_i$ is also GP, 
from Lemma 5.1 in  
\cite{Srinivas:2010:GPO:3104322.3104451}, 
 with probability at least (w.p.a.l.) $1-(d+1)^{-1} \delta$ 
it holds that 
$f^{(1)}_i (\bx) \in Q^{(1)}_{t,i} (\bx)$ for any $i \in [d]$, $t \geq 1$ and $\bx \in \mathcal{X}$.
This implies that w.p.a.l. 
$1-d(d+1)^{-1} \delta$ it holds that $f^{(1)}_i (\bx) \in Q^{(1)}_{t,i} (\bx)$ for all 
$i \in [d]$. 
Similarly, by using the same argument for $f^{(2)}_{jk}$, w.p.a.l. 
$1-d^{-2}(d+1)^{-1} \delta$ the inequality 
$|f^{(2)}_{jk} (\bx) - \mu^{(2)}_{t,jk} (\bx) | \leq \tilde{\gamma}^{1/2}_{t} \sigma^{(2)}_{t,jk} (\bx)$ 
holds for any $j,k \in [d]$, $t \geq1$ and $\bx \in \mathcal{X}$, where $\tilde{\gamma}^{1/2}_t = d^{-1} \gamma^{1/2}_t$. 
Thus, w.p.a.l. $1-(d+1)^{-1} \delta$  above inequalities are simultaneously satisfied. 
Here,  denote the Hessian matrix of $\bx$ by ${\bm{H}}(\bx) = {\bm{M}}_t (\bx) +{\bm{Z}}_t (\bx)$, where 
 the $(j,k)^{\rm th}$ element of ${\bm{M}}_t (\bx)$ is $\mu^{(2)}_{t,jk} (\bx)$ and that of 
 ${\bm{Z}}_{t} (\bx)$  is normal distribution with mean $0$ and variance  $\{\sigma^{(2)}_{t,jk} (\bx) \}^2$.
Therefore, w.p.a.l. $1-(d+1)^{-1} \delta$
 the absolute value of each element in  
${\bm{Z}}_t(\bx)$ is less than $\tilde{\gamma}^{1/2}_{t} \varsigma^{(2)}_t (\bx)$. 
Then, for any ${\bm{a}}$ satisfying  $\| {\bm{a}} \| =1$, it holds that 
\begin{align*}
\hspace{-6mm} \lambda (\bx) = \inf_{{\bm{a}}} {\bm{a}}^\top {\bm{H}} (\bx) {\bm{a}} 
& \geq \inf_{ {\bm{a}} } {\bm{a}}^\top {\bm{M}}_t (\bx) {\bm{a}}   +  \inf_{ {\bm{a}} } {\bm{a}}^\top {\bm{Z}}_t (\bx) {\bm{a}}   \\
&= \lambda_{t} (\bx) + \inf_{ {\bm{a}} } {\bm{a}}^\top {\bm{Z}}_t (\bx) {\bm{a}}  .
\end{align*}
Furthermore, noting that $( |a_1|+ \cdots +|a_d| )^2 \leq d$ and 
$$| {\bm{a}}^\top {\bm{Z}}_t (\bx) {\bm{a}}  | \leq \tilde{\gamma}^{1/2}_{t} \varsigma^{(2)}_t (\bx) 
 ( |a_1|+ \cdots +|a_d| )^2, 
$$
 the inequality $| {\bm{a}}^\top {\bm{Z}}_t (\bx) {\bm{a}}  | \leq \tilde{\gamma}^{1/2}_{t} \varsigma^{(2)}_t (\bx) d = \gamma^{1/2}_t \varsigma^{(2)}_t (\bx)$ holds. 
Hence, we have  
  $\lambda (\bx) \geq \lambda_{t} (\bx)- \gamma^{1/2}_t \varsigma^{(2)}_t (\bx)$. 
Similarly, we also have 
 $\lambda (\bx) \leq \lambda_{t} (\bx)+ \gamma^{1/2}_t \varsigma^{(2)}_t (\bx)$. 
This implies that w.p.a.l. 
 $1-(d+1)^{-1} \delta$ it holds that $\lambda (\bx) \in Q^{(2)}_{t} (\bx)$. 
Therefore, w.p.a.l. 
 $1-\delta$ it holds that $f^{(1)}_i(\bx) \in Q^{(1)}_{t,i} (\bx)$, $i \in [d]$ and 
$\lambda (\bx) \in Q^{(2)}_{t} (\bx)$. 
\end{proof}
Next, we provide an upper bound of $\eta_t$. Let 
\begin{align*}
\tilde{f}^{(1)}_{t,i} (\bx; \zeta) &= f^{(1)}_{t,i} (\bx) - \frac { f_t (\bx +\zeta {\bm{e}}_i ) - f_t(\bx)}{\zeta}, \\
\tilde{f}^{(2)}_{t,jk} (\bx; \zeta) &= f^{(2)}_{t,jk} (\bx) - \frac { f_t (\bx +\zeta  {\bm{e}}_{jk} ) - f_t(\bx +\zeta {\bm{e}}_j)}{\zeta ^2 } 
   -  \frac { f _t(\bx +\zeta  {\bm{e}}_{k} ) - f_t(\bx )}{\zeta ^2 }, 
\end{align*}  
where ${\bm{e}}_{jk} ={\bm{e}}_j + {\bm{e}}_k $ and ${\bm{e}}_{i} $ is a $d$-dimensional vector whose $i$th element is one and  remainders are zeros. 
Assume the following conditions:
\begin{enumerate}
\renewcommand{\labelenumi}{(A\arabic{enumi}).}
\item There exists a positive constant $A_0$ such that for any $\zeta$ satisfying $|\zeta| <A_0$,  $(\bx + \zeta {\bm{e}}_i ) \in D $   and $(\bx +\zeta {\bm{e}}_{jk} ) \in D $ for any $\bx \in \mathcal{X}$ and $i,j,k \in [d]$. \label{enu:A2}
\item  There exists a positive constant $C_0$ such that 
$$
\V [ \tilde{f}^{(1)}_{0,i} (\bx; \zeta) ] \leq | \zeta |  C_0 , \quad  \V [ \tilde{f}^{(2)}_{0,jk} (\bx; \zeta) ] \leq | \zeta |  C_0 ,
$$
 for any $\bx \in \mathcal{X}$, $i,j,k \in [d]$ and $\zeta$ satisfying $| \zeta | < A_0$, where $A_0$ is given in the assumption (A\ref{enu:A2}). \label{enu:A1}
\end{enumerate}
Finally, let ${\rm{I}} ({\bm{y}}_t;f) $ be a mutual information between ${\bm{y}}_t $ and $f$. 
Also let $\kappa _t $ be a maximum information gain after $t$ rounds on $D$, defined by 
$
\kappa_t = \max_{A \subset D;|A|=t}  {\rm{I}} ({\bm{y}}_A;f_A).
$
Then, the following theorem holds:
\begin{theorem}\label{thm:upper-bound}
Let $C_1 = 2 \sigma^2 C_2$, 
$C_2 = \sigma^{-2} / \log (1+ \sigma^{-2} )$,  $R_t = r_1 + \cdots + r_t$, $\tilde{C}_0 > C_0$  and 
\begin{align*}
\tilde{\eta}^2_t = \max \left \{ 
\frac{ 3200 \tilde{C}^2_0 C_1 \beta^3_t \kappa_t}{R_t \epsilon ^{4} }, 
\frac{1250  \tilde{C}^4_0 C_1 \gamma^5_{t} \kappa_t}{R_t \epsilon ^{8}}
\right \}. 
\end{align*}
Then, it holds that 
$
\eta^2_t  \leq \tilde{\eta}^2_t +    \frac{4}{5} \epsilon ^2 .
$
\end{theorem}
The proof is given in Appendix. 
In addition, Srinivas t al. 
\cite{Srinivas:2010:GPO:3104322.3104451}  
 provided the order of $\kappa_t$ for certain kernels under mild conditions. 
For example, in Gaussian kernel, its order is $\mathcal{O} ((\log t)^{d+1}) $. 
Hence, if we set $R_t =  \mathcal{O} (t) $, then $\tilde{\eta}^2_t $ converges to 0, i.e., 
 $\eta^2_T$ satisfies $\eta^2_T < \epsilon ^2$ for some $T$. 


\section{Numerical experiments}\label{sec5}
\begin{figure*}[t]
\begin{center}
\scalebox{1}{
 \begin{tabular}{cccc}
 \includegraphics[width=0.225\textwidth]{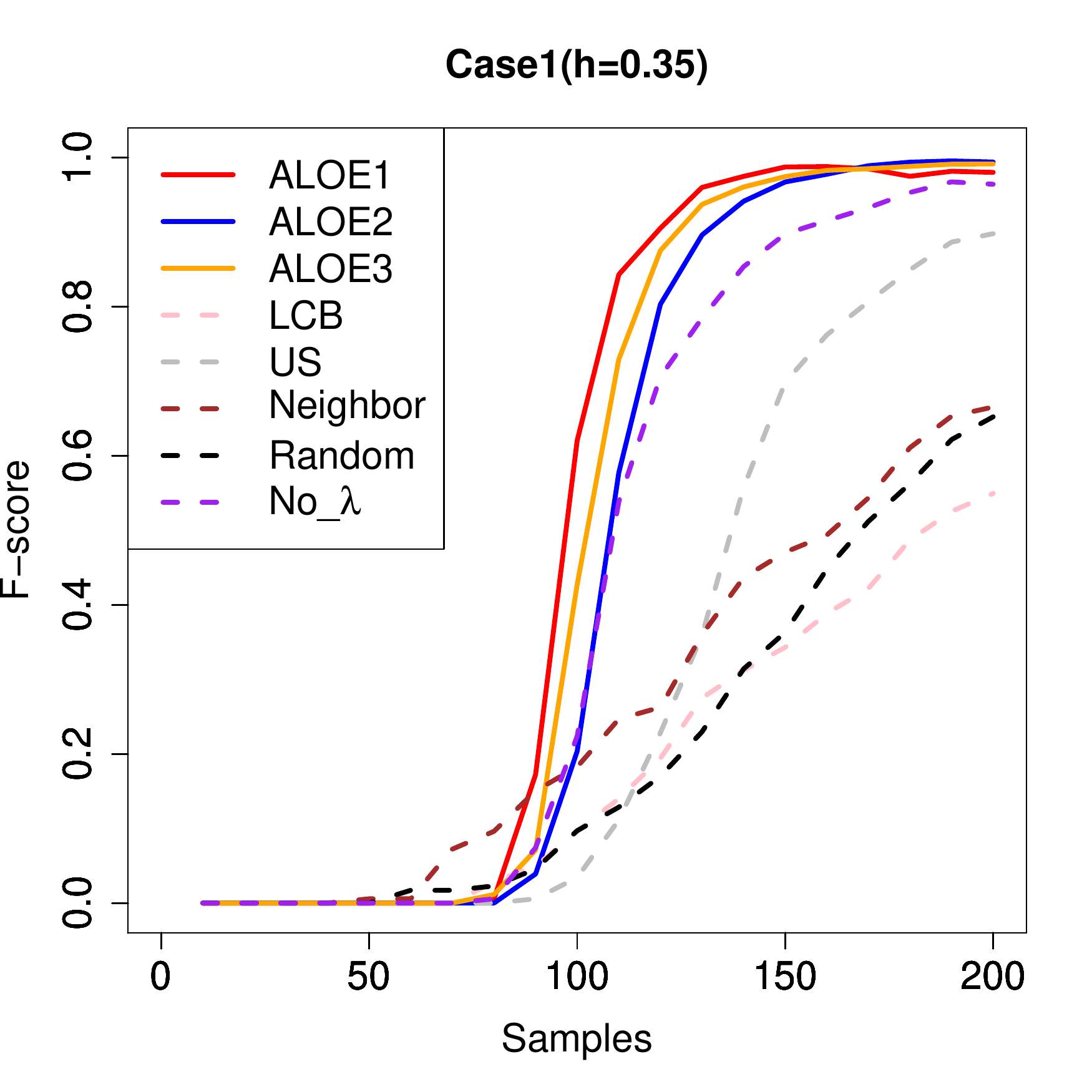} & 
 \includegraphics[width=0.225\textwidth]{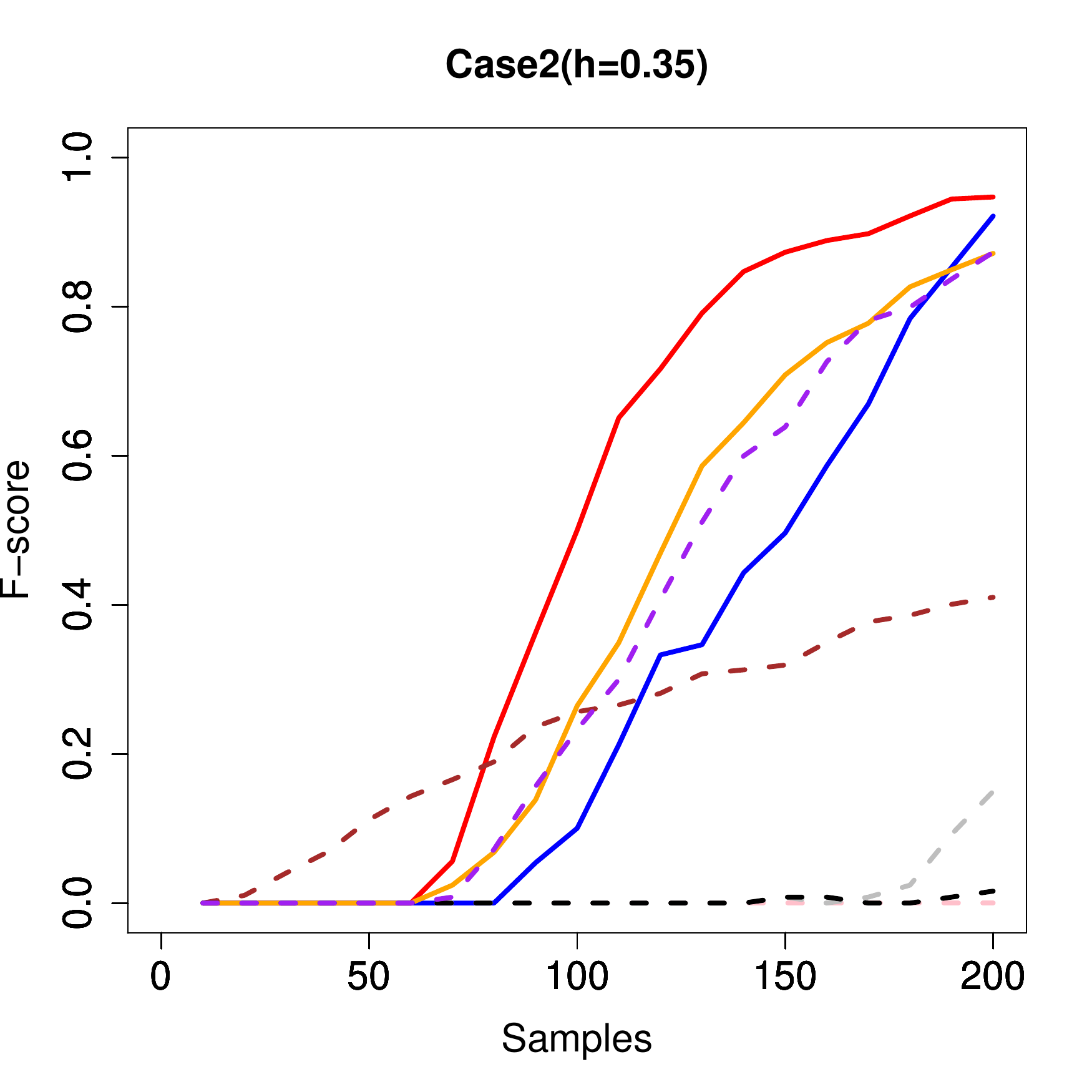} &
 \includegraphics[width=0.225\textwidth]{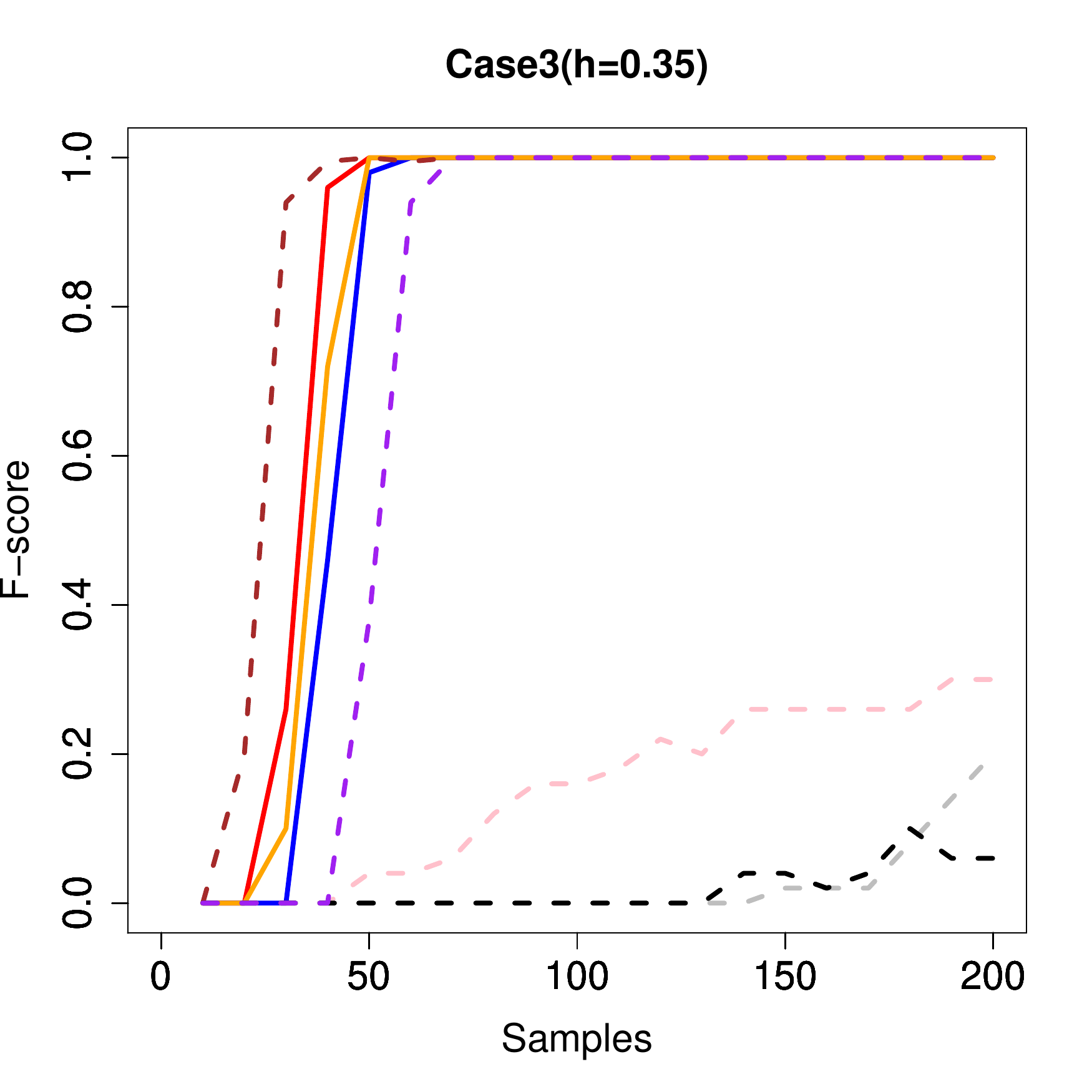} & 
 \includegraphics[width=0.225\textwidth]{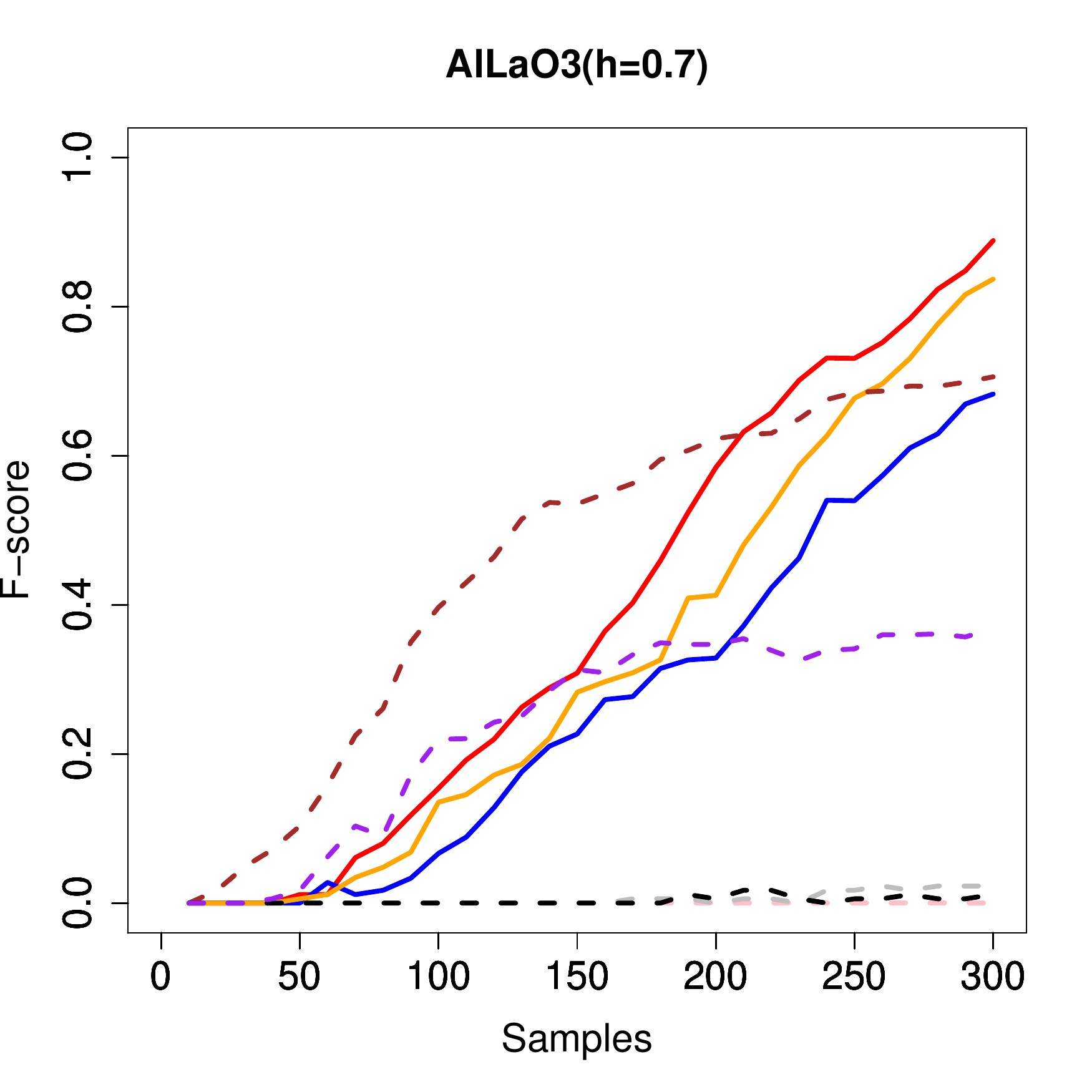}  \\
 \includegraphics[width=0.225\textwidth]{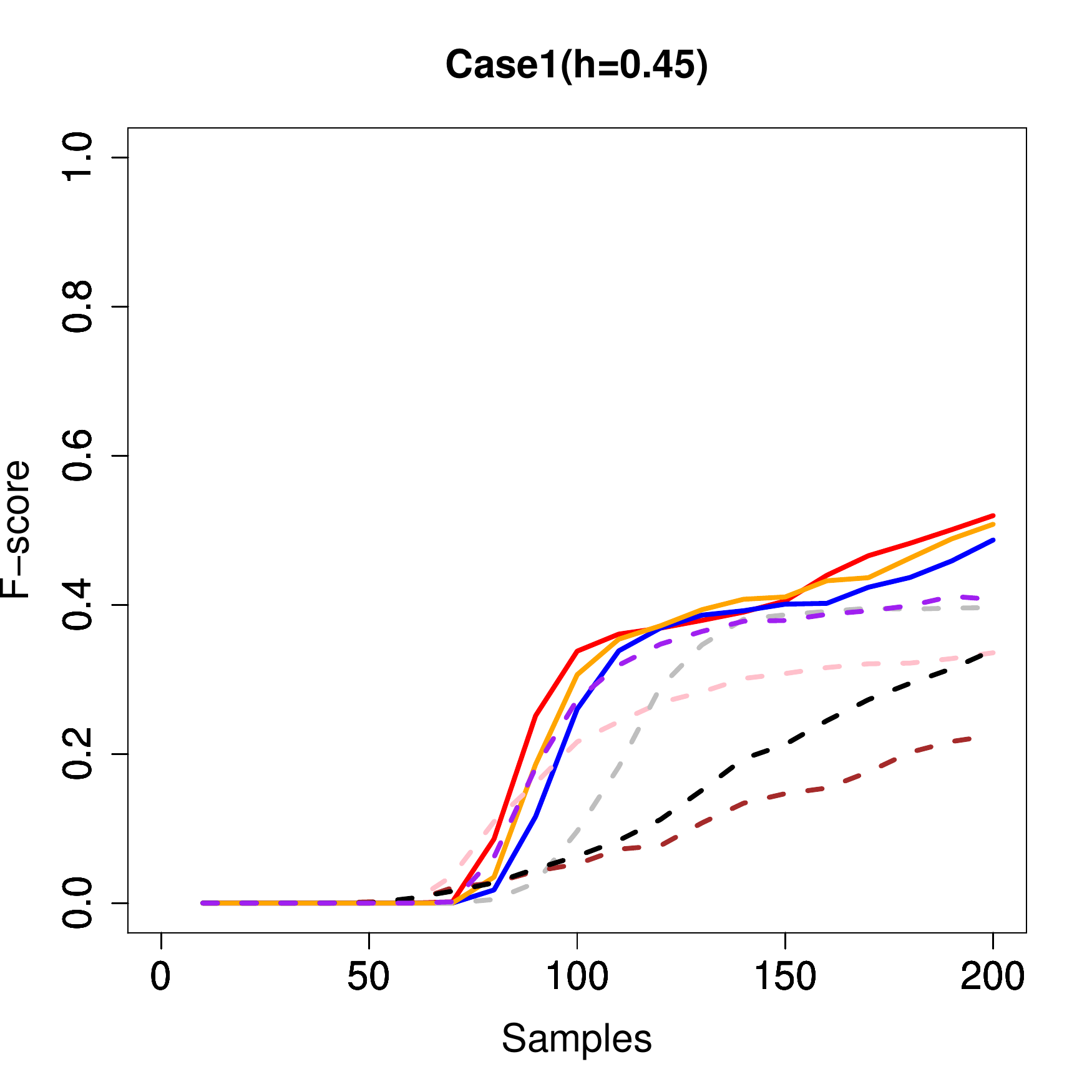} & 
 \includegraphics[width=0.225\textwidth]{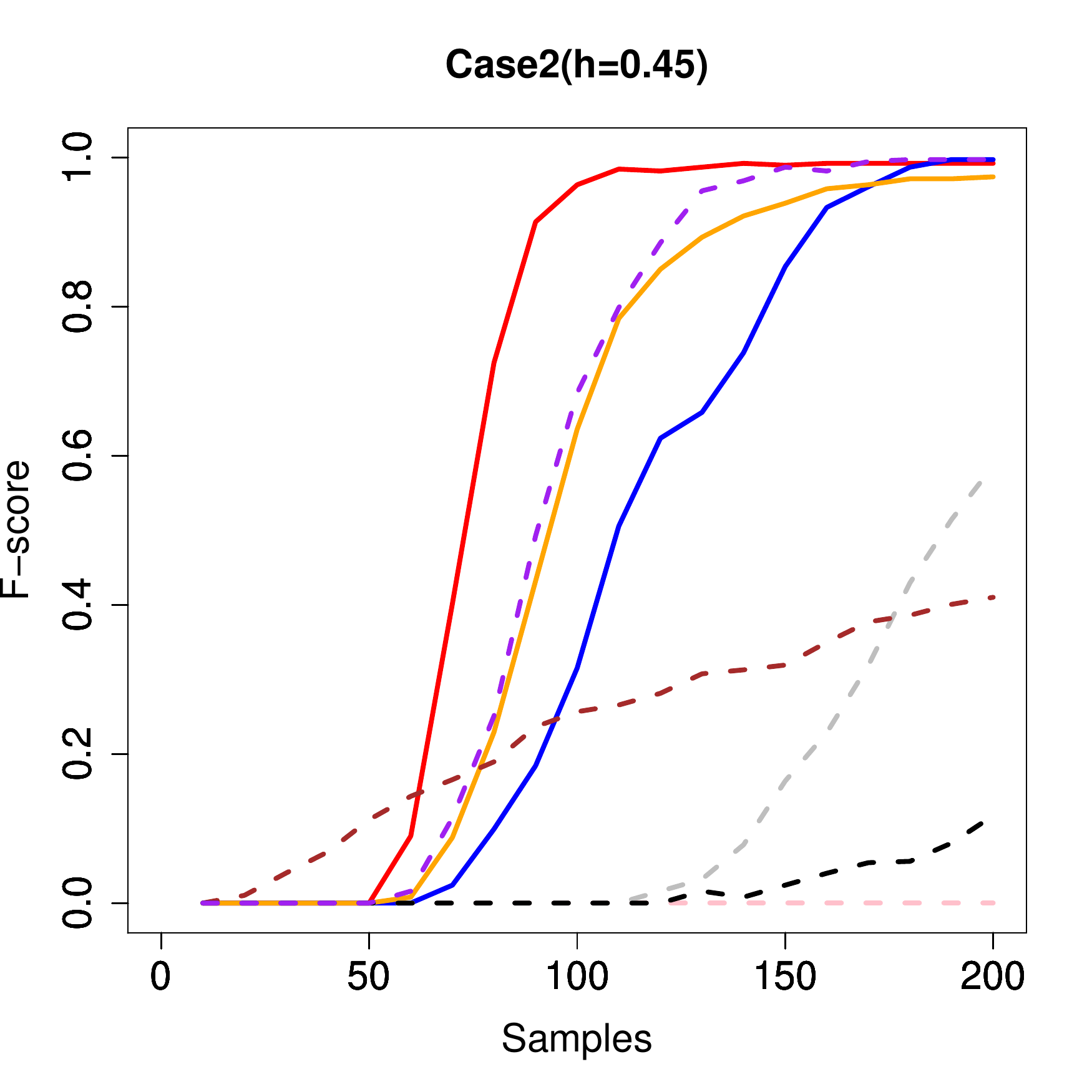} & 
 \includegraphics[width=0.225\textwidth]{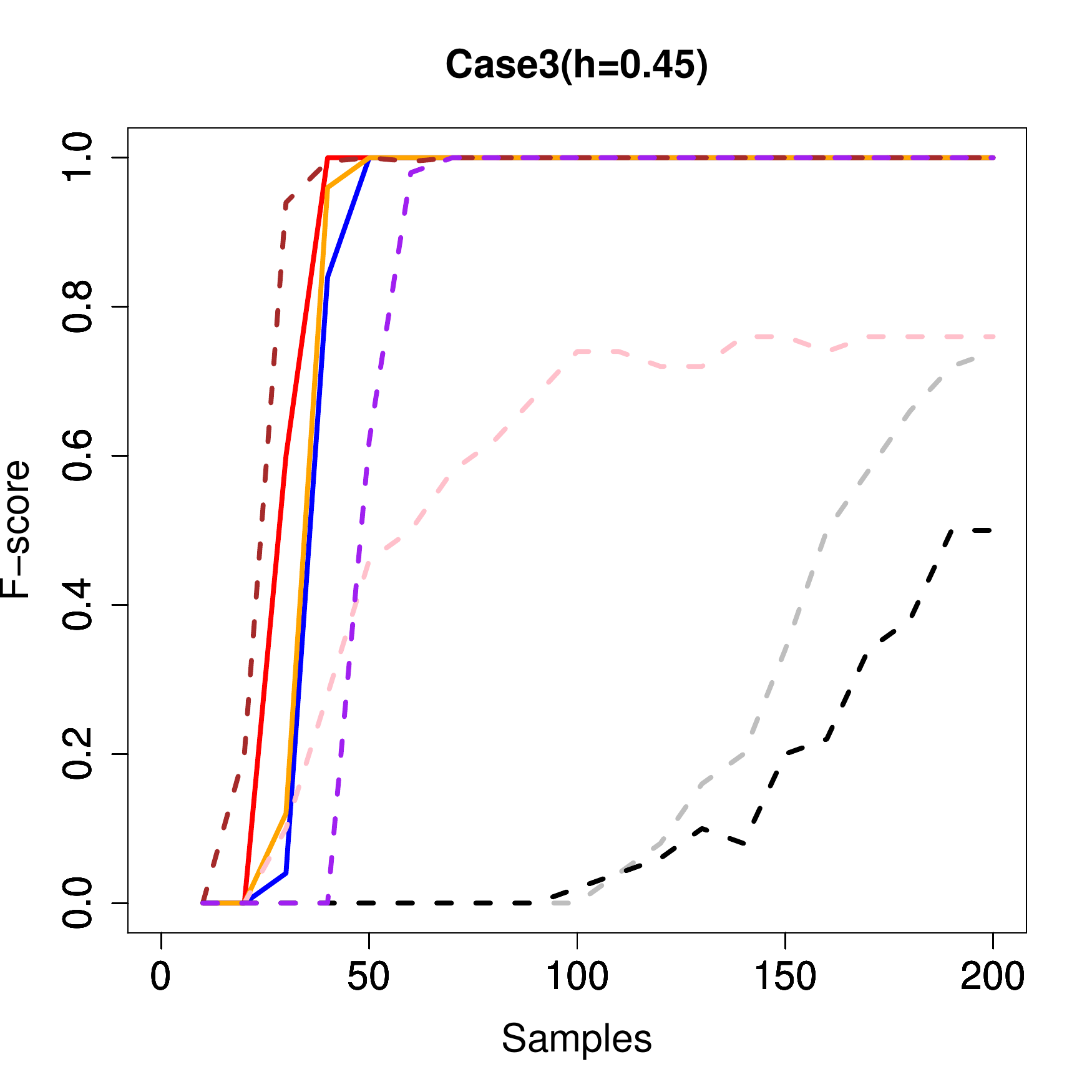} & 
 \includegraphics[width=0.225\textwidth]{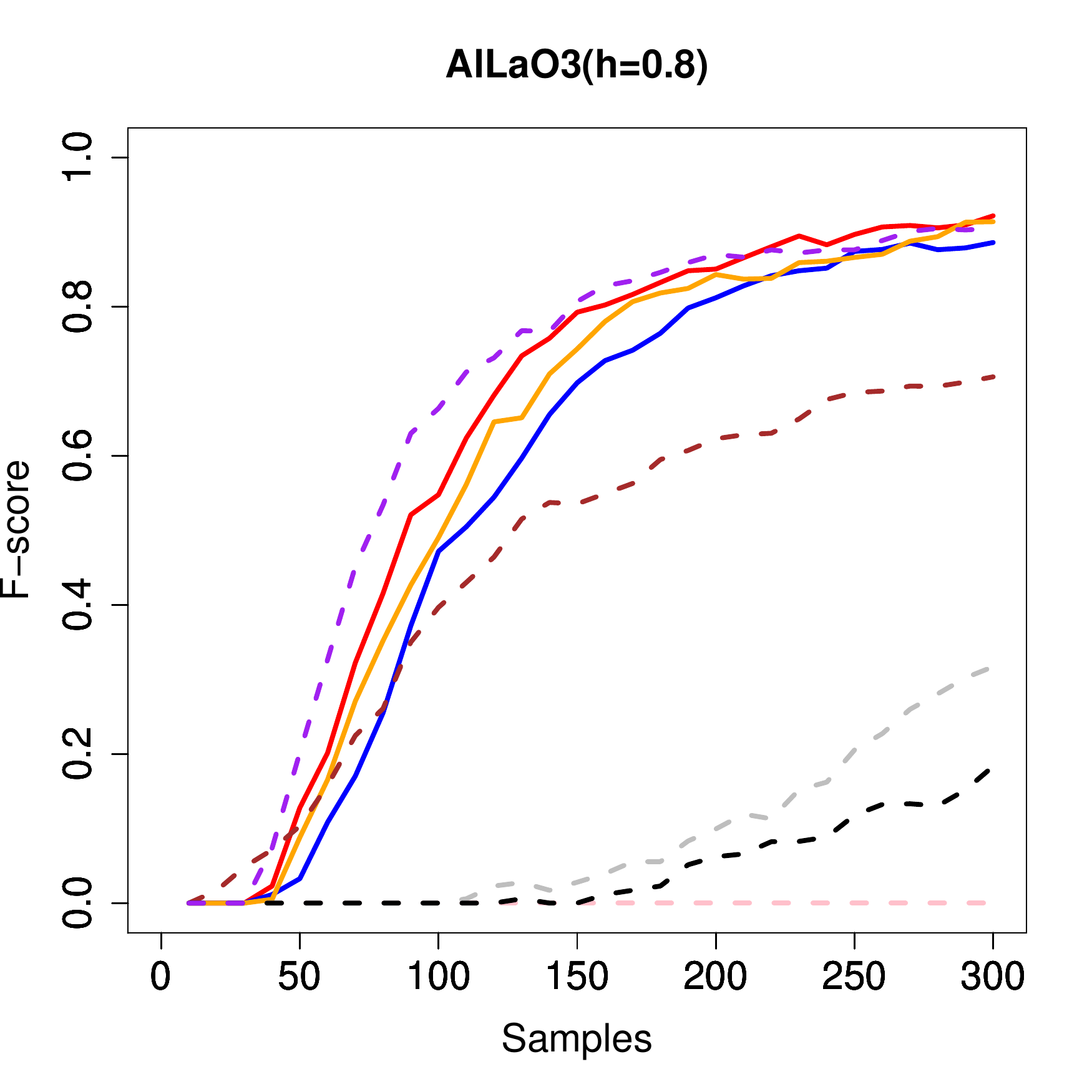}  
 \end{tabular}}
\end{center}
 \caption{Average F-score for each experiment. The bottom left plot indicates the result of real data experiments, while the other 7 plots indicate the results of synthetic function experiments.}
 \label{fig:fscore}
\end{figure*}
We confirm the performance of ALOE by numerical experiments.
We used F-score defined as $F= 2\cdot Pre \cdot Rec/(Pre+Rec)$, where 
precision $Pre$ and recall $Rec$ are given by
$$
Pre= | S \cap \widehat{S}_t | / | \widehat{S}_t |, \ 
Rec= | S \cap \widehat{S}_t | / | S |.
$$
Here, we compared the following seven AFs: 
\begin{description}
\item[(Random):] Random sampling.
\item[(US):] Uncertainty sampling, i.e., we set $r_t=1$ for any $t \geq 1$ in \eqref{eq:AF}.
\item[(LCB):] Lower confidence bound (LCB) $lcb(\bx)$, where $lcb(\bx) = \mu_t (\bx) -3  \sigma_t (\bx)$.
\item[(No\_$\lambda$):] AF for ALOE with $r_t=0$ for any $t \geq1$, and use $U_t= \mathcal{X} \setminus (\bigcap_{i=1}^d G^{(1)}_{t,i}  \cup \bigcup _{i=1}^d \bar{G}^{(1)}_{t,i} )$.
\item[(ALOE1):] AF for ALOE with $r_t=0$ for any $t \geq1$.
\item[(ALOE2):] AF for ALOE with $r_t=1$ if $t$ is a multiple of 5 and otherwise 0.
\item[(ALOE3):] AF for ALOE with $r_t=1$ if $t$ is a multiple of 10 and otherwise 0.
\end{description}
Furthermore, we consider the following as a competitor. 
Let $$
P_t (\bx;\alpha ) = \PR ( f_t (\bx) \leq f_t ( \bx^{(\alpha) }_{l,s} ),\  l \in [d], s \in \{1,-1\} )
$$  
where $ \bx^{(\alpha)}_{l,1}$ satisfies $\bx - \bx^{(\alpha)}_{l,1} = \alpha {\bm{e}}_l$ and 
 $ \bx^{(\alpha)}_{l,-1}$ satisfies $\bx - \bx^{(\alpha)}_{l,-1} = -\alpha {\bm{e}}_l$.
Then, if $\bx $ satisfies $P_t (\bx;0.3) >0.6$, $\bx$ is classified as $\widehat{S}_t$. 
When using this neighborhood based classification rule (Neighbor), we use $Nei_t (\bx) = \sigma^2_t (\bx) P_t (\bx;0.3 )$ 
as the acquisition function, and next input is selected by $\bx_{t+1} = \argmax  Nei_t (\bx)  $. 
In all the experiments , we use  Gaussian kernel 
$k(\bx,\bx')= \sigma^2_f \exp (- \| \bx - \bx' \|^2/L)$. 
Moreover, we assume that error variance $\sigma^2$ is known.

\subsection{Synthetic  function experiments}\label{exp:ex1}
We considered 2-dimensional synthetic  functions. 
Let $\mathcal{A}$ be a grid point set obtained by dividing the interval $[A,B]$ into 40 equal parts, and let 
 $D= \mathcal{A} \times \mathcal{A}$. Define 
$$
\mathcal{X}= \{ \bx=(x_1,x_2) \in D \ | \  x_1 \in [a,b], \ x_2 \in [a,b] \}.
$$
The following three cases were considered: 
\begin{description}
\item[(Case 1):] $f(x_1,x_2) = \sin (x_1) \cos (x_2)$, $A=-\pi/2$, $B=7 \pi /2 $, $a=0$, $b=3 \pi$,  $\sigma^2_f=1$, $L=4.5$. 
\item[(Case 2):] $f(x_1,x_2) =18+ \sum_{s=1}^2  \{ (1/4)x^4_s -(13/3)x^3_s +25x^2_s -56x_s  \} /3 $, 
 $A=-1$, $B=9$, $a=0$, $b=8  $,  $\sigma^2_f=2$, $L=3$. 
\item[(Case 3):] $f(x_1,x_2) = \sum_{s=1}^2   (x_s -4) ^2  $, 
 $A=-1$, $B=9$, $a=0$, $b=8  $,  $\sigma^2_f=2$, $L=3$. 
\end{description}
Furthermore, we set  $\sigma^2=0.005$,   $\beta^{1/2}_t= \gamma^{1/2}_{t} =3$, 
 $\epsilon^{(1)}_1=\epsilon^{(1)}_2 \equiv h \in \{0.35, 0.45 \}$ and  $\epsilon^{(2)}=0.1$. 
Here, At this time, one initial point was randomly determined, and based on each AF, function evaluations were sequentially done up to step 200. 
This was repeated 50 times, and the average of the F-score was calculated (Fig.\ref{fig:fscore}).
The results indicate that ALOE has better performance than other methods.

\subsection{Real data experiments}\label{exp:ex3}
We analyzed the  potential energy (PE) data in inorganic crystal AlLaO3.
The data includes 3-dimensional inputs $ \bx_i \in D $ corresponding to 3-dimensional coordinates and PE  
$ y ^ \ast _ i $, for $ i = 1, \ldots, 5832 $. 
Here, $D$ is given by  $D= \mathcal{A}^3$ and $\mathcal{A}$ is   
a grid point set obtained by dividing the interval $[0,r]$ into 17 equal parts, where  $r \approx 3.6$.
In this experiment, GP was first fitted using the whole data excluding outliers, and the posterior mean at each point is defined as the true function $ f (\bx) $.
We used this to calculate the energy at each point as $ y _ i = f (\bx_i) + \varepsilon_i $. 
Moreover, we defined
 $
\mathcal{X}= \{ \bx=(x_1,x_2,x_3) \in D \ | \  x_s \in [0.4,2], \ s=1,2,3 \}
$. 
 Furthermore, we set 
  $\sigma^2=0.01$, $L=2.5$, $\beta^{1/2}_{t}=4$, $ \gamma^{1/2}_{t} =1$,  
 $\epsilon^{(1)}_1=\epsilon^{(1)}_2=\epsilon^{(1)}_3 \equiv h \in \{0.7,0.8\}$ and  $\epsilon^{(2)}=1.2$. 
In addition, it is known that there are 6 local minimum points in $\mathcal{X}$ from the domain knowledge in material science.
Therefore, these six points are defined as the members of $S$.
Here, one initial point was randomly selected, and based on each AF, function evaluations was iterated up to step 300.
This was repeated 50 times, and the average of F-score was calculated (see the bottom right plot in Fig.\ref{fig:fscore}).
The results indicate that the performance of ALOE is better than other competitors as in the previous synthetic experiments.




\section{Conclusion}\label{sec6}
In this paper, we proposed an AL method called ALOE for enumerating local minima using GP derivatives.
From the theoretical results and numerical experiments, the usefulness of ALOE was confirmed.
%

\section*{Acknowledgments}
This work was partially supported by MEXT KAKENHI (17H00758, 16H06538), JST CREST (JPMJCR1302, JPMJCR1502), RIKEN Center for Advanced Intelligence Project, and JST support program for starting up innovation-hub on materials research by information integration initiative.







\bibliography{myref}

\begin{thebibliography}{10}

\bibitem{NIPS2007_3189}
Edwin~V Bonilla, Kian~M. Chai, and Christopher Williams.
\newblock Multi-task gaussian process prediction.
\newblock In {\em Advances in Neural Information Processing Systems 20}, pages
  153--160. Curran Associates, Inc., 2008.

\bibitem{NIPS2011_4221}
David~K Duvenaud, Hannes Nickisch, and Carl~E. Rasmussen.
\newblock Additive gaussian processes.
\newblock In {\em Advances in Neural Information Processing Systems 24}, pages
  226--234. Curran Associates, Inc., 2011.

\bibitem{ghosal2006posterior}
Subhashis Ghosal and Anindya Roy.
\newblock Posterior consistency of gaussian process prior for nonparametric
  binary regression.
\newblock {\em The Annals of Statistics}, pages 2413--2429, 2006.

\bibitem{Gotovos:2013:ALL:2540128.2540322}
Alkis Gotovos, Nathalie Casati, Gregory Hitz, and Andreas Krause.
\newblock Active learning for level set estimation.
\newblock In {\em Proceedings of the Twenty-Third International Joint
  Conference on Artificial Intelligence}, pages 1344--1350, 2013.

\bibitem{Hennig:2012:ESI:2503308.2343701}
Philipp Hennig and Christian~J. Schuler.
\newblock Entropy search for information-efficient global optimization.
\newblock {\em J. Mach. Learn. Res.}, 13(1):1809--1837, June 2012.

\bibitem{pmlr-v77-nguyen17a}
Vu~Nguyen, Sunil Gupta, Santu Rana, Cheng Li, and Svetha Venkatesh.
\newblock Regret for expected improvement over the best-observed value and
  stopping condition.
\newblock In {\em Proceedings of the Ninth Asian Conference on Machine
  Learning}, volume~77, pages 279--294, 2017.

\bibitem{papoulis2002probability}
Athanasios Papoulis and S~Unnikrishna Pillai.
\newblock {\em Probability, random variables, and stochastic processes}.
\newblock Tata McGraw-Hill Education, 2002.

\bibitem{Rasmussen:2005:GPM:1162254}
Carl~Edward Rasmussen and Christopher K.~I. Williams.
\newblock {\em Gaussian Processes for Machine Learning (Adaptive Computation
  and Machine Learning)}.
\newblock The MIT Press, 2005.

\bibitem{settles2009active}
Burr Settles.
\newblock Active learning literature survey.
\newblock Computer Sciences Technical Report 1648, University of
  Wisconsin--Madison, 2009.

\bibitem{shahriari2016taking}
Bobak Shahriari, Kevin Swersky, Ziyu Wang, Ryan~P Adams, and Nando De~Freitas.
\newblock Taking the human out of the loop: A review of bayesian optimization.
\newblock {\em Proceedings of the IEEE}, 104(1):148--175, 2016.

\bibitem{NIPS2002_2287}
E.~Solak, R.~Murray-smith, W.~E. Leithead, D.~J. Leith, and Carl~E. Rasmussen.
\newblock Derivative observations in gaussian process models of dynamic
  systems.
\newblock In {\em Advances in Neural Information Processing Systems 15}, pages
  1057--1064. MIT Press, 2003.

\bibitem{Srinivas:2010:GPO:3104322.3104451}
Niranjan Srinivas, Andreas Krause, Sham Kakade, and Matthias Seeger.
\newblock Gaussian process optimization in the bandit setting: No regret and
  experimental design.
\newblock In {\em Proceedings of the 27th International Conference on Machine
  Learning}, pages 1015--1022, 2010.

\bibitem{sui2015safe}
Yanan Sui, Alkis Gotovos, Joel Burdick, and Andreas Krause.
\newblock Safe exploration for optimization with gaussian processes.
\newblock In {\em International Conference on Machine Learning}, pages
  997--1005, 2015.

\bibitem{Wang:2016:BOB:3013558.3013569}
Ziyu Wang, Frank Hutter, Masrour Zoghi, David Matheson, and Nando De~Freitas.
\newblock Bayesian optimization in a billion dimensions via random embeddings.
\newblock {\em J. Artif. Int. Res.}, 55(1):361--387, January 2016.

\bibitem{NIPS2017_7111}
Jian Wu, Matthias Poloczek, Andrew~G Wilson, and Peter Frazier.
\newblock Bayesian optimization with gradients.
\newblock In {\em Advances in Neural Information Processing Systems 30}, pages
  5267--5278. Curran Associates, Inc., 2017.

\end{thebibliography}
\bibliographystyle{plain}

\section*{Appendix}

\setcounter{section}{0}
\renewcommand{\thesection}{\Alph{section}}

\section{Proof of Theorem \ref{thm:upper-bound}}
From the definition of $\eta_t$, $\eta^2_t$ can be written by 
\begin{align}
\eta^2_t = \max  \{ \max_{\bx \in \mathcal{X}} 4 \beta _{t} \{ \sigma^{(1)}_{t,1} (\bx) \}^2, \ldots , \max_{ \bx \in \mathcal{X}} 4 \beta _{t}\{  \sigma^{(1)}_{t,d} (\bx) \}^2,  
\max_{\bx \in \mathcal{X}} \gamma _{t}  \{ \varsigma^{(2)}_t (\bx) \}^2 \}   .   \label{eq:etasquare}
\end{align}
Here, for any $i \in [d]$, $t  \geq 1$ and $\bx \in \mathcal{X}$, let 
\begin{align*}
\Delta f ^{(1)} _{t,i} (\bx ; \zeta ) = \frac { f_t (\bx +\zeta {\bm{e}}_i ) - f_t(\bx)}{\zeta}. 
\end{align*}
Then, it holds that 
\begin{align}
\{ \sigma^{(1)}_{t,i} (\bx) \}^2 = \V [ f^{(1)}_{t,i}  (\bx ) ] &= \V [ \Delta f ^{(1)} _{t,i} (\bx ; \zeta ) + \tilde{f}^{(1)}_{t,i} (\bx;\zeta) ]  \nonumber \\ 
&=  \V [ \Delta f ^{(1)} _{t,i} (\bx ; \zeta ) ]  
+2 {\rm{Cov}} [ \Delta f ^{(1)} _{t,i} (\bx ; \zeta ) ,\tilde{f}^{(1)}_{t,i} (\bx;\zeta) ]
+ \V[\tilde{f}^{(1)}_{t,i} (\bx;\zeta) ] \nonumber \\
& \leq 2 \V [ \Delta f ^{(1)} _{t,i} (\bx ; \zeta ) ]  
+ 2 \V[\tilde{f}^{(1)}_{t,i} (\bx;\zeta) ], \label{eq:upp1}
\end{align}
where last inequality is derived by $ 2{\rm{Cov}} [X,Y] \leq \V[X]+\V[Y]$. 
Note that $(f(\bx_1), \ldots, f(\bx_t) , \tilde{f}^{(1)}_i (\bx  ;\zeta ) )^\top $ is distributed as a multivariate normal distribution. 
Thus, from the definition of the conditional variance in the multivariate normal distribution, from the assumption (A\ref{enu:A1}) we get 
$ \V[\tilde{f}^{(1)}_{t,i} (\bx;\zeta) ]   \leq  \V[\tilde{f}^{(1)}_{0,i} (\bx;\zeta) ]$. 
Hence, by substituting this inequality into \eqref{eq:upp1}, we obtain 
\begin{align}
\{ \sigma^{(1)}_{t,i} (\bx) \}^2  
 \leq 2 \V [ \Delta f ^{(1)} _{t,i} (\bx ; \zeta ) ]  
+ 2 \V[\tilde{f}^{(1)}_{0,i} (\bx;\zeta) ]  
\leq 2 \V [ \Delta f ^{(1)} _{t,i} (\bx ; \zeta ) ]  
+ 2 |\zeta | C_0
. \label{eq:upp2}
\end{align}
Furthermore, the variance  $ \V [ \Delta f ^{(1)} _{t,i} (\bx ; \zeta ) ]  $ satisfies the following inequality:
\begin{align}
\V [ \Delta f ^{(1)} _{t,i} (\bx ; \zeta ) ] 
&= 
\frac{  \V[f_t(\bx + \zeta {\bm{e}}_i )] -2 \Cov [f_t(\bx + \zeta {\bm{e}}_i ),f_t (\bx)] + \V [f_t (\bx) ]     }{\zeta^2}  \nonumber \\
& \leq \frac   {  2  \V[f_t(\bx + \zeta {\bm{e}}_i )]  +2 \V [f_t (\bx) ]       }{\zeta^2}, \label{eq:vardelta}
\end{align}
where last inequality is derived by  $- 2{\rm{Cov}} [X,Y] \leq \V[X]+\V[Y]$. 
Therefore, by using \eqref{eq:upp2} and \eqref{eq:vardelta}, we have   
\begin{align}
4 \beta_t \{ \sigma^{(1)}_{t,i} (\bx) \}^2  \leq
8 \beta_t \frac   {  2  \V[f_t(\bx + \zeta {\bm{e}}_i )]  +2 \V [f_t (\bx) ]       }{\zeta^2} 
+ 8 \beta_t |\zeta| C_0 . \label{eq:upp3}
\end{align}
Here, let $\tilde{C}_0 $ be a positive number satisfying $ \tilde{C}_0 > C_0$, $\epsilon^2 (10 \beta_1 \tilde{C}_0 ) ^{-1} < A_0$ 
and  $(2/5) \epsilon^2 (  \gamma_{1} \tilde{C}_0 ) ^{-1} < A_0$. 
Then, we set 
\begin{equation}
\zeta = \frac{\epsilon^2}{10  \beta_t  \tilde{C}_0}. \label{eq:zetanodef}
\end{equation}
Thus, noting that $\beta_1 \leq \beta_t$, we obtain $|\zeta| < A_0$ and 
\begin{equation}
 8 \beta_t |\zeta| C_0 = \frac{ 8 \epsilon^2 C_0 }{10 \tilde{C}_0}  \leq \frac{4}{5} \epsilon^2. \label{eq:epue}
\end{equation}
Hence, by substituting \eqref{eq:zetanodef} and \eqref{eq:epue} into  \eqref{eq:upp3}, we get 
$$
4 \beta_t \{ \sigma^{(1)}_{t,i} (\bx) \}^2  \leq 
800 \beta^3_t \tilde{C}^2_0  \epsilon^{-4}   (2  \V[f_t(\bx + \zeta {\bm{e}}_i )]  +2 \V [f_t (\bx) ]    ) 
+ \frac{4}{5} \epsilon^2.
$$
Here, define 
$$
\bx^\star_{t+1} = \argmax _{ \bx \in D } \sigma^2_t (\bx ) .
$$
Moreover, from the definition of $\zeta$ and the assumption (A\ref{enu:A2}), it holds that $\bx \in D$, $(\bx +\zeta {\bm{e}}_i) \in D$.
Therefore, it holds that 
\begin{equation}
2  \V[f_t(\bx + \zeta {\bm{e}}_i )]  +2 \V [f_t (\bx) ] 
\leq 4 \sigma^2_t ( \bx^\star_{t+1} ). \nonumber 
\end{equation}
Thus, we have 
\begin{align}
4 \beta_t \{ \sigma^{(1)}_{t,i} (\bx) \}^2  \leq 
3200 \beta^3_t \tilde{C}^2_0  \epsilon^{-4}  \sigma^2_t ( \bx^\star_{t+1} )
+ \frac{4}{5} \epsilon^2. \label{eq:3200}
\end{align} 
Next, suppose that $t_1,\ldots , t_l$ are positive integers satisfying  $1 \leq t_1 \leq \cdots \leq t_l \leq t$ and $r_{t_s} =1$. 
Then, noting that the monotonicity of posterior variances and the definition of \eqref{eq:AF}, 
we obtain 
  $$\sigma^2_t (x^\star_{t+1}) \leq \sigma^2_{t_l } (x^\star_{t+1}) \leq \sigma^2_{t_l} (x_{t_l+1})$$ and 
$$\sigma^2_{t_1} (x_{t_1+1} ) +\cdots + \sigma^2_{t_l} (x_{t_l+1} ) \geq  t_l \sigma^2_{t_l } (x_{t_l+1} ) = R_t \sigma^2_{t_l} (x_{t_l+1}) .$$
Therefore, from Lemma 5.3 in  Srinivas et al. \cite{Srinivas:2010:GPO:3104322.3104451}, we have 
\begin{align}
\sigma^2_t (x^\star_{t+1} ) \leq R^{-1}_t \sum_{s=1}^l \sigma^2_{t_s} (x_{t_s+1} ) & \leq R^{-1} _t \sum_{i=1}^t \sigma^2_i (x_{i+1})   
 \leq C_1 R^{-1}_t \kappa_t. \label{eq:ig}
\end{align}
Thus, substituting \eqref{eq:ig} into \eqref{eq:3200} we obtain 
$$
4 \beta_t \{ \sigma^{(1)}_{t,i} (\bx) \}^2  \leq 
3200 \beta^3_t \tilde{C}^2_0  \epsilon^{-4}  C_1 \kappa_t R^{-1}_t 
+ \frac{4}{5} \epsilon^2. 
$$
This implies that 
\begin{equation}
\max_{ \bx \in \mathcal{X} } 4 \beta_t \{ \sigma^{(1)}_{t,i} (\bx) \}^2  \leq 
3200 \beta^3_t \tilde{C}^2_0  \epsilon^{-4}  C_1 \kappa_t R^{-1}_t 
+ \frac{4}{5} \epsilon^2. \label{eq:last1}
\end{equation}

Similarly, for any $j,k \in [d]$, $t  \geq 1$ and $\bx \in \mathcal{X}$, let 
\begin{align*}
\Delta f ^{(2)} _{t,jk} (\bx ; \tilde{\zeta} ) = \frac { f_t (\bx +\tilde{\zeta}  {\bm{e}}_{jk} ) - f_t(\bx +\tilde{\zeta} {\bm{e}}_j) 
 -  f _t(\bx +\tilde{\zeta}  {\bm{e}}_{k} ) + f_t(\bx )}{\tilde{\zeta} ^2 }.
\end{align*}
Then, using  same arguments we get 
\begin{align*}
\gamma_{t} \{ \sigma^{(2)}_{t,jk} (\bx) \} ^2 & \leq 2 \gamma_{t} \V [ \Delta f^{(2)}_{t,jk} (\bx;\tilde{\zeta}) ] + 2 \gamma_{t} |\tilde{\zeta}| C_0 , \\
\V [ \Delta f^{(2)}_{t,jk} (\bx; \tilde{\zeta}) ]  & \leq \frac{  4 \V [ f_t (\bx +\tilde{\zeta}  {\bm{e}}_{jk} )] + 4 \V[ f_t(\bx +\tilde{\zeta} {\bm{e}}_j) ] 
 +4 \V[  f _t(\bx +\tilde{\zeta}  {\bm{e}}_{k} )]  +4 \V[ f_t(\bx ) ] }{\zeta^4}.
\end{align*}
Furthermore, we set $ \tilde{\zeta} = (2/5) \epsilon  ^2 ( \gamma_{t}  \tilde{C}_0 ) ^{-1}$. 
Thus, noting that $|\tilde{\zeta} | <A_0$, we have 
$$
2 \gamma_{t} |\tilde{\zeta}| C_0 \leq \frac{4}{5} \epsilon ^2
$$
 and 
\begin{align*}
\V [ \Delta f^{(2)}_{t,jk} (\bx; \tilde{\zeta}) ]  & \leq \frac{  4 \V [ f_t (\bx +\tilde{\zeta}  {\bm{e}}_{jk} )] + 4 \V[ f_t(\bx +\tilde{\zeta} {\bm{e}}_j) ] 
 +4 \V[  f _t(\bx +\tilde{\zeta}  {\bm{e}}_{k} )]  +4 \V[ f_t(\bx ) ] }{\tilde{\zeta}^4} \\
& \leq 
 \frac{ 16 \sigma^2_t (\bx^\star_{t+1} ) }{\tilde{\zeta} ^4} = 5^4 \epsilon ^{-8}  \tilde{C}^4_0 \gamma ^4_{t}   \sigma^2_t (\bx^\star_{t+1} ).
\end{align*}
Therefore, from \eqref{eq:ig} it holds that 
$$
\gamma_{t} \{ \sigma^{(2)}_{t,jk} (\bx) \} ^2 \leq \frac{ 1250 \tilde{C}^4_0 C_1 \gamma^5_{t} \kappa_t   }{R_t \epsilon ^8} + \frac{4}{5} \epsilon^2. 
$$
Hence, noting that $   \{ \varsigma^{(2)}_{t} (\bx) \} ^2=    \max _ {j,k \in [d] } \{ \sigma^{(2)}_{t,jk} (\bx) \} ^2$, we obtain 
\begin{equation}
\max_{ \bx \in \mathcal{X}   } \gamma_{t} \{ \varsigma^{(2)}_{t} (\bx) \} ^2 \leq \frac{ 1250 \tilde{C}^4_0 C_1 \gamma^5_{t} \kappa_t   }{R_t \epsilon ^8} + \frac{4}{5} \epsilon^2. \label{eq:last2}
\end{equation}
Finally, by substituting \eqref{eq:last1} and \eqref{eq:last2} into \eqref{eq:etasquare}, we get Theorem \ref{thm:upper-bound}. 
\qed

\section{Local minima identification for infinite set $\mathcal{X}$}
In this section, we consider the case that $\mathcal{X}$ is infinite.
Let $\mathcal{X}$ be an infinite set, and let $\mathcal{X}^\star $ be a finite subset of $\mathcal{X}$. 
In addition, we   assume that $D$ is a compact and convex set. 
Moreover, we may assume $D  \subset [0,r]^d$, without loss of generality. 
Here, for each point $\bx \in \mathcal{X}^\star$, we define $Q^{(1)}_{t,i} (\bx)$ and $Q^{(2)}_{t} (\bx)$. 
 Similarly, using accuracy parameters $\epsilon^{(1)}_i >0$ and $\epsilon^{(2)}>0$, we define 
\begin{align}
{G}^{(1),\star}_{t,i} &= \{ \bx \in \mathcal{X}^\star \ | \ - (1+1/d)\epsilon^{(1)}_i < l^{(1)}_{t,i} (\bx) \land  \ u^{(1)}_{t,i} (\bx) < (1+1/d) \epsilon^{(1)}_i \},  \nonumber \\
\bar{{G}}^{(1),\star}_{t,i} &= \{ \bx \in \mathcal{X}^\star  \ | \ \epsilon^{(1)}_i/d \leq l^{(1)}_{t,i} (\bx) \lor \ u^{(1)}_{t,i} (\bx) \leq -\epsilon^{(1)}_i/d \}, \nonumber \\
{H}^{(2),\star}_{t} &= \{ \bx \in \mathcal{X}^\star  \ | \ l^{(2)} _{t} > -\epsilon ^{(2)} \} , \nonumber \\
\bar{H}^{(2),\star}_{t} &= \{ \bx \in \mathcal{X}^\star  \ | \ u^{(2)} _{t} < \epsilon ^{(2)} \} \nonumber
\end{align}
and 
\begin{equation}
\begin{split}
\widetilde{S}_t&= {H}^{(2),\star}_{t} \cap   \bigcap_{i=1}^d {G}^{(1),\star}_{t,i}, \\
\widetilde{\bar{S}}_t&= \bar{H}^{(2),\star}_{t} \cup   \bigcup_{i=1}^d \bar{G}^{(1),\star}_{t,i}. \\
\end{split}\label{eq:bunruiinfinite}
\end{equation}
Moreover, for each ${\bm{a}} \in \mathcal{X}$, let $[{\bm{a}}] $ be a point in $\mathcal{X}^\star$ closest to ${\bm{a}}$.  
Then, from \eqref{eq:bunruiinfinite} we define the estimated sets $\widehat{S}_t$ and $\widehat{\bar{S}}_t$ as follows:
\begin{definition}[Estimated sets $\widehat{S}_t$ and $\widehat{\bar{S}}_t$ for infinite $\mathcal{X}$]
Estimated sets $\widehat{S}_t$ and $\widehat{\bar{S}}_t$ of $S$ and $\mathcal{X} \setminus S $ are respectively defined as  
\begin{equation}
\begin{split}
\widehat{S}_t&= \{    \bx \in \mathcal{X} \ | \  [\bx] \in  \widetilde{S}_t \} \\
\widehat{\bar{S}}_t&=\{    \bx \in \mathcal{X} \ | \  [\bx] \in  \widetilde{\bar{S}}_t \}.   \\
\end{split}\label{eq:bunruiinfinitedush}
\end{equation}
\end{definition}
Furthermore, we define the acquisition function $b_t (\bx)$ as follows:
\begin{definition}[Function $b_t (\bx)$ based on predicted  violatios for infinite $\mathcal{X}$]\label{def:proposedAFinfinite}
Define 
\begin{align}
\tilde{\bx}^\star_t &= \argmax_{\bx \in \tilde{U}_t }   \sum_{i=1}^d V^{(1)}_{t,i} (\bx), \nonumber \\
     \tilde{U}_t &= 
 \mathcal{X}^\star \setminus (\widetilde{S}_t \cup \widetilde{\bar{S}}_t )
. \nonumber 
\end{align}
 Then, the function 
 $b_t ( \bx)$ is defined by
\begin{align}
b_t (\bx) =   \sum_{i=1}^d \left ( V^{(1)}_{t,i} (\tilde{\bx}^\star) - V^{(1)}_{t,i} (\tilde{\bx}^\star;\bx)  \right ). \label{eq:teianAtilde}
\end{align}
\end{definition}
Finally, the flow of the proposed method when $\mathcal{X}$ is infinite is shown in Algorithm \ref{ALG2}.

\begin{algorithm}[httb]                  
\caption{Local minima identification for infinite $\mathcal{X}$}         
\label{ALG2}                          
\begin{algorithmic}[1]                  
\REQUIRE 
 Initial training data,  
GP prior $\mathcal{G}\mathcal{P} (0,k(\bx,\bx'))$
\ENSURE Estimated sets $\widehat{S}$ and $ \widehat{\bar{S}}$
\STATE $\widetilde{S}_0 \gets \emptyset$, $\widetilde{\bar{S}}_0\gets \emptyset$, $\tilde{U}_0 \gets \mathcal{X}^\star$
\STATE $t \gets 1$
\WHILE{$\tilde{U}_{t-1} \neq \emptyset$}
	\STATE $\widetilde{S}_t \gets \widetilde{S}_{t-1}$, $\widetilde{\bar{S}}_t\gets \widetilde{\bar{S}}_{t-1}$, $\tilde{U}_t \gets \tilde{U}_{t-1}$
	\FORALL{$ \bx \in \mathcal{X}^\star$}
	\STATE Compute confidence intervals $Q^{(1)}_{t,i}$ and  $Q^{(2)}_t$ from GP derivatives
	\ENDFOR
	\STATE 
Compute ${G}^{(1),\star}_{t,i} $, $\bar{G}^{(1),\star}_{t,i}$, ${H}^{(2),\star}_{t}$ and  $\bar{H}^{(2),\star}_{t}$ for each\ $i \in [d]$ 
	\FORALL{$ \bx \in \mathcal{X}^\star$}
		\IF{$\bx \in {H}^{(2),\star}_{t} \cap \bigcap_{i=1}^d {G}^{(1),\star}_{t,i}  $} 
		\STATE $\widetilde{S}_t \gets \widetilde{S}_{t-1} \cup \{ \bx \}$, $\tilde{U}_t \gets \tilde{U}_{t-1} \setminus \{ \bx \}$
		 \ELSIF{$\bx \in \bar{H}^{(2),\star}_{t} \cup \bigcup_{i=1}^d \bar{G}^{(1),\star}_{t,i}  $} 
		\STATE $\widetilde{\bar{S}}_t \gets \widetilde{\bar{S}}_{t-1} \cup \{ \bx \}$, $\tilde{U}_t \gets \tilde{U}_{t-1} \setminus \{ \bx \}$
		 \ENDIF
	\ENDFOR
	\STATE Compute   $\bx_{t}$ from \eqref{eq:AF}  and \eqref{eq:teianAtilde}
\STATE $y_t \gets f({\bm{x}}_t ) +\varepsilon_t$
\STATE $t \gets t+1$
\ENDWHILE
\STATE Compute $\widehat{S} _{t-1}$     and  $\widehat{\bar{S}} _{t-1} $ from \eqref{eq:bunruiinfinitedush}
\STATE $\widehat{S} \gets \widehat{S}_{t-1}$, $\widehat{\bar{S}} \gets \widehat{\bar{S}}_{t-1}$
\end{algorithmic}
\end{algorithm}

\section{Theoretical results for infinite $\mathcal{X}$}
 In this section, we provide   the theorem on the performance and convergence of Algorithm \ref {ALG2}.
 Hereafter, instead of the assumptions (A\ref{enu:A2}) and (A\ref{enu:A1}), we assume the following conditions: 
\begin{enumerate}
\renewcommand{\labelenumi}{(C\arabic{enumi}).}
\item There exists a positive constant $A^\star_0$ such that for any $\zeta$ satisfying $|\zeta| <A^\star_0$,  $(\bx + \zeta {\bm{e}}_i ) \in D $   and $(\bx +\zeta {\bm{e}}_{jk} ) \in D $ for any $\bx \in \mathcal{X}^\star$ and $i,j,k \in [d]$. \label{enu:C2}
\item  There exists a positive constant $C^\star_0$ such that 
$$
\V [ \tilde{f}^{(1)}_{0,i} (\bx; \zeta) ] \leq | \zeta |  C^\star_0 , \quad  \V [ \tilde{f}^{(2)}_{0,jk} (\bx; \zeta) ] \leq | \zeta |  C^\star_0 ,
$$
 for any $\bx \in \mathcal{X}^\star$, $i,j,k \in [d]$ and $\zeta$ satisfying $| \zeta | < A^\star_0$, where $A^\star_0$ is given in the assumption (C\ref{enu:C2}). \label{enu:C11}
\end{enumerate}
Furthermore, we also assume the following condition.
\begin{enumerate}
\renewcommand{\labelenumi}{(C\arabic{enumi}).}
\setcounter{enumi}{2}
\item There exists positive constants $a$ and $b$ such that 
\begin{align*}
\PR  \left (
\sup_{\bx \in \mathcal{X}} | f^{(2)}_{ij} (\bx) | \geq L 
\right ) \leq a e^{-(L/b)^2}, \q \f i,j \in [d]
\end{align*}
and 
\begin{align*}
\PR  \left (
\sup_{\bx \in \mathcal{X}} | f^{(3)}_{ijk} (\bx) | \geq L 
\right ) \leq a e^{-(L/b)^2}, \q \f i,j ,k\in [d],
\end{align*}
where $f^{(3)}_{ijk} (\bx) $ is given by $\partial f^{(2)}_{ij} (\bx) / \partial x_k $. \label{enu:C1}
\end{enumerate}
Furthermore,  let $\mathcal{X}^\star$ be a set which has $ (\tau_ \epsilon )^d$ elements, and let 
\begin{equation}
\| \bx -  [\bx ] \|_1 \leq rd /\tau_\epsilon, \q \f \bx \in \mathcal{X}. \label{eq:SA}
\end{equation}
Then, the following theorem holds:
\begin{theorem}\label{thm:seido2}
Let   $\epsilon^{(1)}_1,\ldots, \epsilon^{(1)}_d , \epsilon^{(2)}$ be positive numbers, and let $\epsilon =  \min\{\epsilon^{(1)}_1,\ldots,\epsilon^{(1)}_d ,\epsilon^{(2)} \}$. 
For any 
$\delta \in (0,1)$, define $L=b ( \log \{ a(d^2+d^3)/\delta \} )^{1/2}$, 
$\tau_\epsilon =  \lceil d^2 r \epsilon ^{-1} L \rceil $, 
 \begin{align*}
\beta_t &= 2\log ( (d+1)  \pi^2 t^2/(6 \delta) )  + 2d \log (  \lceil b d^2 r \epsilon ^{-1} \sqrt{\log (a(d^2+d^3)/\delta)} \rceil    ) , \\
\gamma_{t} &= 2d^2 \log (d^2 (d+1)   \pi^2 t^2 /(6 \delta )) +  2d^2 \log (  \lceil b d^2 r \epsilon ^{-1} \sqrt{\log (a(d^2+d^3)/\delta)} \rceil    ) 
, \\
\eta_t &= \max \{ \max_{\bx \in \mathcal{X}^\star} 2\beta^{1/2}_{t} \sigma^{(1)}_{t,1} (\bx), \ldots , \max_{ \bx \in \mathcal{X}^\star} 2\beta^{1/2}_{t} \sigma^{(1)}_{t,d} (\bx), 
\max_{\bx \in \mathcal{X}^\star} \gamma^{1/2}_{t} \varsigma^{(2)}_t (\bx)  \}.
\end{align*}
Then, 
Algorithm \ref{ALG2} completes classification after at least the minimum positive integer $T$ trials that satisfy the following inequality 
$\eta^2_T \leq \epsilon^2$. 
Moreover, with probability at least 
 $1-2 \delta$, for any $t \geq 1$, $\bx \in \mathcal{X}$ and $i \in [d]$ it holds that 
$$
\bx \in \widehat{S}_t \Rightarrow -(1+2/d) \epsilon^{(1)}_1 < f^{(1)}_i (\bx) < (1+2/d) \epsilon^{(1)}_i \land \ \lambda (\bx) > -2 \epsilon^{(2)} 
$$ and 
$$
\bx \in \widehat{\bar{S}}_t \Rightarrow   f^{(1)}_i (\bx) \neq 0 \lor \ \lambda (\bx) < 2 \epsilon^{(2)} .
$$
\end{theorem}

In order to prove Theorem \ref{thm:seido2}, we consider following  lemmas and corollaries. 
\begin{lemma}\label{lem:C1}
From the assumption (C\ref{enu:C1}), it holds that 
\begin{align*}
\PR \left (
\f i,j,k \in [d], \f \bx \in \mathcal{X}, \ |f^{(2)}_{ij} (\bx) | < L \land |f^{(3)}_{ijk} (\bx) | <L 
\right ) \geq 1- (d^2+d^3) a e^{-(L/b)^2}.
\end{align*}
\end{lemma}

\begin{proof}
The proof is given by using the same arguments as in the Appendix A.2. of  Srinivas et al. \cite{Srinivas:2010:GPO:3104322.3104451}. 
\end{proof}

Furthermore, 
from  Appendix A.2. of  Srinivas et al. \cite{Srinivas:2010:GPO:3104322.3104451} and 
  Lemma \ref{lem:C1},  we obtain the following corollary.
\begin{cor}\label{cor:C1}
With probability at least $1-(d^2+d^3)ae^{-(L/b)^2 }$, it holds that 
\begin{align*}
| f^{(1)}_i (\bx) -f^{(1)}_i (\bx ')| \leq L \| \bx -\bx' \|_1, \q \f \bx,\bx' \in \mathcal{X}  , \ \f i \in [d]
\end{align*}
and
\begin{align*}
| f^{(2)}_{jk} (\bx) -f^{(2)}_{jk} (\bx ')| \leq L \| \bx -\bx' \|_1, \q \f \bx,\bx' \in \mathcal{X}  , \ \f j,k \in [d].
\end{align*}
\end{cor}

Then, the following lemma holds:
\begin{lemma}\label{lem:C2}
For any $\delta \in (0,1)$, let $L=b ( \log \{ a(d^2+d^3)/\delta \} )^{1/2}$ and 
$\tau_\epsilon =  \lceil d^2 r \epsilon ^{-1} L \rceil $. 
Then, with probability at least $1-\delta$,  
\begin{align}
| f^{(1)}_i (\bx) -f^{(1)}_i ([\bx])| \leq \epsilon/d, \q \f \bx \in \mathcal{X}  , \ \f i \in [d] , \label{eq:C2e1}
\end{align}
and
\begin{align}
| f^{(2)}_{jk} (\bx) -f^{(2)}_{jk} ([\bx ])| \leq\epsilon/d, \q \f \bx \in \mathcal{X}  , \ \f j,k \in [d]. \label{eq:C2e2}
\end{align}
\end{lemma}
\begin{proof}
From Corollary \ref{cor:C1} and \eqref{eq:SA}, w.p.a.l. $1-\delta$ we have 
\begin{align*}
| f^{(1)}_i (\bx) -f^{(1)}_i ([\bx])| \leq Lrd/\tau_\epsilon, \q \f \bx \in \mathcal{X}  , \ \f i \in [d]
\end{align*}
and
\begin{align*}
| f^{(2)}_{jk} (\bx) -f^{(2)}_{jk} ([\bx ])| \leq Lrd/\tau_\epsilon, \q \f \bx \in \mathcal{X}  , \ \f j,k \in [d].
\end{align*}
Hence, noting that $\tau^{-1}_\epsilon \leq \epsilon (Lrd^2)^{-1}$, we obtain Lemma \ref{lem:C2}.
\end{proof}

Using these we provide a proof of Theorem \ref{thm:seido2}.
\begin{proof}
First, for each $\bx \in \mathcal{X}^\star$, when the inequality on $\eta_T $ holds, 
the lengths of  $Q^{(1)}_{T,i} (\bx)$ and $Q^{(2)}_{t} (\bx)$ are less than 
$\epsilon^{(1)}_i$ and $2\epsilon^{(2)}$, respectively. 
Hence, from  classification rules,  all points in $\mathcal{X}^\star$ are classified. 
Moreover, by replacing $|\mathcal{X}|$ of $\beta_t$ and $\gamma_{t} $ in Theorem \ref{thm:seido} with $| \mathcal{X}^\star|$, 
we get 
 \begin{align*}
\beta_t &   = 2\log ( (d+1)  \pi^2 t^2/(6 \delta) )  + 2d \log (  \lceil b d^2 r \epsilon ^{-1} \sqrt{\log (a(d^2+d^3)/\delta)} \rceil    )   , \\
\gamma_{t} &= 2d^2 \log (d^2 (d+1)   \pi^2 t^2 /(6 \delta )) +  2d^2 \log (  \lceil b d^2 r \epsilon ^{-1} \sqrt{\log (a(d^2+d^3)/\delta)} \rceil    ) .
\end{align*}
Therefore, from Theorem \ref{thm:seido}, w.p.a.l. $1-\delta$, for any $\bx \in \mathcal{X}^\star$,  $t \geq 1$ and $i \in [d] $ it holds that 
\begin{align}
\bx \in \widetilde{S}_t \Rightarrow -(1+1/d) \epsilon^{(1)}_1 < f^{(1)}_i (\bx) <(1+1/d)  \epsilon^{(1)}_i \land \ \lambda  (\bx) > -\epsilon^{(2)}  , \label{eq:note1} \\
\bx \in \widetilde{\bar{S}}_t \Rightarrow  | f^{(1)}_i (\bx) | \geq \epsilon^{(1)}_i/d  \lor \ \lambda  (\bx) < \epsilon^{(2)}  . \label{eq:note2}
\end{align}
Here, by combining  Lemma \ref{lem:C2}, w.p.a.l. $1-2\delta$,  the equations \eqref{eq:C2e1}, \eqref{eq:C2e2}, \eqref{eq:note1} and 
\eqref{eq:note2} hold. 
In addition, letting $\tilde{{\bm{H}}} (\bx)$ be a matrix the $(j,k)^{\rm th}$ element of which is given by 
$ f^{(2)}_{jk} (\bx) -f^{(2)}_{jk} ([\bx ])$, for any ${\bm{a}} $ satisfying $\| a \| =1$ we obtain 
\begin{align}
| {\bm{a}}^\top \tilde{{\bm{H}}} (\bx) {\bm{a}} | \leq \epsilon , \nonumber 
\end{align}
when \eqref{eq:C2e2} holds. 
Thus, noting that ${\bm{H}} (\bx) = {\bm{H}} ([\bx ]) + \tilde{{\bm{H}}} (\bx)$, we get 
\begin{equation}
\lambda  (\bx) \leq \lambda  ([\bx ] ) + \epsilon,   \q 
\lambda  (\bx) \geq \lambda  ([\bx ] ) - \epsilon . \label{eq:end}
\end{equation}
Hence, noting that $\epsilon \leq \epsilon^{(1)}_i $ and $\epsilon \leq \epsilon^{(2)} $, from the definition of $\widehat{S}_t$, we have 
\begin{align*}
\bx \in \widehat{S}_t  & \Rightarrow  [\bx] \in \widetilde{S}_t  \land \eqref{eq:C2e1} \land \eqref{eq:end} \\
&\Rightarrow 
-(1+1/d) \epsilon^{(1)}_1 < f^{(1)}_i ([\bx]) <(1+1/d)  \epsilon^{(1)}_i \land \ \lambda  ([\bx]) > -\epsilon^{(2)}  \land \eqref{eq:C2e1} \land \eqref{eq:end}  \\
&\Rightarrow 
-(1+2/d) \epsilon^{(1)}_1 < f^{(1)}_i (\bx) <(1+2/d)  \epsilon^{(1)}_i \land \ \lambda  ( \bx ) > -2 \epsilon^{(2)}  .
\end{align*}
Similarly, by using the argument we obtain 
$$
\bx \in \widehat{\bar{S}}_t \Rightarrow   f^{(1)}_i (\bx) \neq 0 \lor \ \lambda  (\bx) < 2 \epsilon^{(2)}  .
$$
\end{proof}

In addition, the following theorem holds:
\begin{theorem}\label{thm:upper-bound2}
Let $C_1 = 2 \sigma^2 C_2$, 
$C_2 = \sigma^{-2} / \log (1+ \sigma^{-2} )$,  $R_t = r_1 + \cdots + r_t$, $\tilde{C}^\star_0 > C^\star_0$  and 
\begin{align*}
\tilde{\eta}^2_t = \max \left \{ 
\frac{ 3200 (\tilde{C}^\star_0)^2 C_1 \beta^3_t \kappa_t}{R_t \epsilon ^{4} }, 
\frac{1250  (\tilde{C}^\star_0)^4 C_1 \gamma^5_{t} \kappa_t}{R_t \epsilon ^{8}}
\right \}. 
\end{align*}
Then, it holds that 
$
\eta^2_t  \leq \tilde{\eta}^2_t +    \frac{4}{5} \epsilon ^2 .
$
\end{theorem}
\begin{proof}
The proof is given by using the same argument as the proof of Theorem \ref{thm:upper-bound}.
\end{proof}

\section{Sufficient conditions for assumptions}
 In this section, we provide   sufficient conditions for the assumptions. 
First, the following Lemma holds: 
\begin{lemma}\label{lem:sc1}
Let $D$ be a compact set, and let ${\rm int} (D) $ be an interior set of $D$. 
Suppose that $\mathcal{X} $ is a finite subset of ${\rm int} (D) $. 
Assume that the kernel function $k(\bx,\bx^\prime)$ is a five times continuously differentiable function. 
Then, the assumptions (A\ref{enu:A2}) and (A\ref{enu:A1}) hold.
\end{lemma}
\begin{proof}
For any $\bx \in \mathcal{X} \subset {\rm int} (D)$, noting that ${\rm int} (D)$ is the open set,  there exists a positive number $\delta_\bx $ such that $\mathscr{N} ( \bx; \delta_\bx )   \subset {\rm int} (D) \subset D$, where $\mathscr{N} ( \bx; \delta_\bx )$ is the $\delta_\bx$-neighborhood of $\bx$. 
Here, since $\mathcal{X}$ is the finite set, we can define $A_0 := \min_{\bx \in \mathcal{X} } \{ \delta_\bx \} $. 
Hence, for any $\bx \in \mathcal{X}$, it holds that  $\mathscr{N} (\bx; A_0 ) \subset D$. 
Thus, this implies that the assumption (A\ref{enu:A2}) holds. 

Next, for any $\bx \in \mathcal{X}$, $i \in [d]$ and $\zeta$ satisfying $|\zeta| <A_0$, the variance of $\tilde{f}^{(1)}_{0,i} (\bx ; \zeta) $
is given by 
\begin{align}
\V[\tilde{f}^{(1)}_{0,i} (\bx ; \zeta) ] &=\{ \sigma^{(1)}_{0,i} (\bx ) \} ^2 + 
\frac{  k(\bx + \zeta {\bm{e}}_i, \bx + \zeta {\bm{e}}_i) -  k(\bx + \zeta {\bm{e}}_i, \bx  )    -k(\bx , \bx + \zeta {\bm{e}}_i)     +k(\bx,\bx)        }{\zeta^2} \nonumber \\
&\q - \frac{   \Cov[f^{(1)}_{0,i}(\bx),f_0 (\bx +\zeta {\bm{e}}_i )] - \Cov[f^{(1)}_{0,i}(\bx),f_0 (\bx )]    }{\zeta} \nonumber \\
&\q - \frac{   \Cov[f_0 (\bx +\zeta {\bm{e}}_i ),f^{(1)}_{0,i}(\bx)] - \Cov[f_0 (\bx ),f^{(1)}_{0,i}(\bx)]    }{\zeta}  \label{eq:ap4-D1}.
\end{align}
Here, we put $\bx + \zeta {\bm{e}}_i = {\bx }^\ast $. Then, $k(\bx, \bx + \zeta {\bm{e}}_i )$ can be written by 
$k(\bx, \bx + \zeta {\bm{e}}_i ) = k( \bx, \bx^\ast)$, and using Taylor's expansion we have 
\begin{align}
k(\bx, \bx + \zeta {\bm{e}}_i ) &= k( \bx, \bx^\ast) \nonumber \\
&= k(\bx, \bx) + \left . \frac{\partial k(\bx, \bx^\ast)}{\partial x^\ast_i} \right |_{\bx^\ast = \bx }    \zeta +
\frac{1}{2}   \left . \frac{\partial ^2  k(\bx, \bx^\ast)}{\partial x^\ast_i \partial x^\ast_i } \right |_{\bx^\ast = \bx } \zeta ^2  
+ \frac{1}{6}   \left . \frac{\partial ^3  k(\bx, \bx^\ast)}{\partial x^\ast_i \partial x^\ast_i \partial x^\ast_i} \right |_{\bx^\ast = \bx^\star_1 } \zeta ^3,  \label{eq:ap4-D2}
\end{align}  
where $\bx^\star_1$ is a point on a line segment connecting $\bx^\ast $ and $\bx$.
Similarly, we obtain 
\begin{align}
k(\bx + \zeta {\bm{e}}_i, \bx + \zeta {\bm{e}}_i ) &= k( \bx + \zeta {\bm{e}}_i, \bx^\ast) \nonumber \\
&= k(\bx + \zeta {\bm{e}}_i, \bx) + \left . \frac{\partial k(\bx + \zeta {\bm{e}}_i, \bx^\ast)}{\partial x^\ast_i} \right |_{\bx^\ast = \bx }    \zeta +
\frac{1}{2}   \left . \frac{\partial ^2  k(\bx + \zeta {\bm{e}}_i, \bx^\ast)}{\partial x^\ast_i \partial x^\ast_i } \right |_{\bx^\ast = \bx } \zeta ^2  \nonumber \\
& \q + \frac{1}{6}   \left . \frac{\partial ^3  k(\bx + \zeta {\bm{e}}_i, \bx^\ast)}{\partial x^\ast_i \partial x^\ast_i \partial x^\ast_i} \right |_{\bx^\ast = \bx^\star_2 } \zeta ^3,  \label{eq:ap4-D3}
\end{align}  
where $\bx^\star_2$ is also a point on a line segment connecting $\bx^\ast $ and $\bx$.
Moreover, we get 
\begin{align}
\left . \frac{\partial  k(\bx + \zeta {\bm{e}}_i, \bx^\ast)}{\partial x^\ast_i   } \right |_{\bx^\ast = \bx }  
&= \left . \frac{\partial   k(\bx, \bx^\ast)}{\partial x^\ast_i  } \right |_{\bx^\ast = \bx }  
+ 
\left . \frac{\partial  ^2  k({\bm{a}}, \bx^\ast)}{\partial a_i \partial x^\ast_i  } \right |_{\bx^\ast = \bx , {\bm{a}}=\bx} \zeta 
+ \frac{1}{2}  \left . \frac{\partial  ^3  k({\bm{a}}, \bx^\ast)}{\partial a_i \partial a_i \partial x^\ast_i  } \right |_{\bx^\ast = \bx , {\bm{a}}=\bx^\star_3} \zeta ^2  \nonumber \\
&=             \left . \frac{\partial   k(\bx, \bx^\ast)}{\partial x^\ast_i  } \right |_{\bx^\ast = \bx }  
+ 
\{\sigma^{(1)}_{0,i} (\bx) \} ^2 \zeta 
+ \frac{1}{2}  \left . \frac{\partial  ^3  k({\bm{a}}, \bx^\ast)}{\partial a_i \partial a_i \partial x^\ast_i  } \right |_{\bx^\ast = \bx , {\bm{a}}=\bx^\star_3} \zeta ^2                                           \label{eq:ap4-D3-2}
\end{align}
and
\begin{align}
\left . \frac{\partial ^2  k(\bx + \zeta {\bm{e}}_i, \bx^\ast)}{\partial x^\ast_i \partial x^\ast_i } \right |_{\bx^\ast = \bx }  
&= \left . \frac{\partial ^2  k(\tilde{\bx}, \bx^\ast)}{\partial x^\ast_i \partial x^\ast_i } \right |_{\bx^\ast = \bx } \nonumber \\
&= \left . \frac{\partial ^2  k(\bx, \bx^\ast)}{\partial x^\ast_i \partial x^\ast_i } \right |_{\bx^\ast = \bx } 
+ 
\left . \frac{\partial ^3  k(\tilde{\bx}, \bx^\ast)}{\partial \tilde{x}_i \partial x^\ast_i \partial x^\ast_i } \right |_{\bx^\ast = \bx , \tilde{\bx} = \bx^\star_4} \zeta.   \label{eq:ap4-D4}
\end{align}
 Here, noting that 
$$
\Cov[f^{(1)}_{0,i}(\bx),f_0 (\bx +\zeta {\bm{e}}_i )] =  \left . \frac{\partial    k({\bm{a}}, \bx +\zeta {\bm{e}}_i)}{\partial  a_i } \right |_{{\bm{a}}= \bx },
$$
 the covariance $\Cov[f^{(1)}_{0,i}(\bx),f_0 (\bx +\zeta {\bm{e}}_i )] $ can be expressed as 
\begin{align}
\Cov[f^{(1)}_{0,i}(\bx),f_0 (\bx +\zeta {\bm{e}}_i )] &=  \left . \frac{\partial    k({\bm{a}}, \bx +\zeta {\bm{e}}_i)}{\partial  a_i } \right |_{{\bm{a}}= \bx } = \left . \frac{\partial    k({\bm{a}}, \bx ^\ast) }{\partial  a_i } \right |_{{\bm{a}}= \bx }  \nonumber \\
&=
\left . \frac{\partial    k({\bm{a}}, \bx)  }{\partial  a_i } \right |_{{\bm{a}}= \bx } + 
\left . \frac{\partial ^2   k({\bm{a}}, \bx^\ast )  }{\partial  a_i  \partial x^\ast_i} \right |_{{\bm{a}}= \bx, \bx^\ast = \bx } \zeta +
\frac{1}{2} \left . \frac{\partial ^3   k({\bm{a}}, \bx^\ast )  }{\partial  a_i  \partial x^\ast_i \partial x^\ast_i } \right |_{{\bm{a}}= \bx, \bx^\ast = \bx^\star_5 } \zeta ^2 \nonumber \\
&=\Cov[f^{(1)}_{0,i}(\bx),f_0 (\bx   )]  + \{\sigma^{(1)}_{0,i} (\bx) \}^2 \zeta + \frac{1}{2} \left . \frac{\partial ^3   k({\bm{a}}, \bx^\ast )  }{\partial  a_i  \partial x^\ast_i \partial x^\ast_i } \right |_{{\bm{a}}= \bx, \bx^\ast = \bx^\star_5 } \zeta ^2. \label{eq:ap4-D5}
\end{align}
By using the same argument, we also have 
\begin{align}
\Cov[f_0 (\bx +\zeta {\bm{e}}_i ),f^{(1)}_{0,i}(\bx)]  
=\Cov[f_0 (\bx   ),f^{(1)}_{0,i}(\bx)]  + \{\sigma^{(1)}_{0,i} (\bx) \}^2 \zeta + \frac{1}{2} \left . \frac{\partial ^3   k( \bx^\ast,{\bm{a}} )  }{\partial x^\ast_i \partial x^\ast_i  \partial  a_i  } \right |_{  \bx^\ast = \bx^\star_6,            {\bm{a}}= \bx } \zeta ^2. \label{eq:ap4-D6}
\end{align}
Therefore, using \eqref{eq:ap4-D3}, \eqref{eq:ap4-D3-2} and \eqref{eq:ap4-D4}, we obtain 
\begin{align}
k(\bx + \zeta {\bm{e}}_i, \bx + \zeta {\bm{e}}_i ) 
- k(\bx + \zeta {\bm{e}}_i, \bx) &= \left . \frac{\partial k(\bx + \zeta {\bm{e}}_i, \bx^\ast)}{\partial x^\ast_i} \right |_{\bx^\ast = \bx }    \zeta +
\frac{1}{2}   \left . \frac{\partial ^2  k(\bx + \zeta {\bm{e}}_i, \bx^\ast)}{\partial x^\ast_i \partial x^\ast_i } \right |_{\bx^\ast = \bx } \zeta ^2  \nonumber \\
&\q   + \frac{1}{6}   \left . \frac{\partial ^3  k(\bx + \zeta {\bm{e}}_i, \bx^\ast)}{\partial x^\ast_i \partial x^\ast_i \partial x^\ast_i} \right |_{\bx^\ast = \bx^\star_2 } \zeta ^3 \nonumber \\
&=   \left . \frac{\partial   k(\bx, \bx^\ast)}{\partial x^\ast_i  } \right |_{\bx^\ast = \bx }  \zeta
+ 
\{\sigma^{(1)}_{0,i} (\bx) \} ^2 \zeta ^2
+ \frac{1}{2}  \left . \frac{\partial  ^3  k({\bm{a}}, \bx^\ast)}{\partial a_i \partial a_i \partial x^\ast_i  } \right |_{\bx^\ast = \bx , {\bm{a}}=\bx^\star_3} \zeta ^3  \nonumber \\
&\q + \frac{1}{2} \left . \frac{\partial ^2  k(\bx, \bx^\ast)}{\partial x^\ast_i \partial x^\ast_i } \right |_{\bx^\ast = \bx } \zeta^2
+  \frac{1}{2} 
\left . \frac{\partial ^3  k(\tilde{\bx}, \bx^\ast)}{\partial \tilde{x}_i \partial x^\ast_i \partial x^\ast_i } \right |_{\bx^\ast = \bx , \tilde{\bx} = \bx^\star_4} \zeta^3 \nonumber \\
& \q +   \frac{1}{6}   \left . \frac{\partial ^3  k(\bx + \zeta {\bm{e}}_i, \bx^\ast)}{\partial x^\ast_i \partial x^\ast_i \partial x^\ast_i} \right |_{\bx^\ast = \bx^\star_2 } \zeta ^3. \label{eq:ap4-D8}
\end{align}
Thus, by combining \eqref{eq:ap4-D2} and \eqref{eq:ap4-D8}, we have 
\begin{align}
&\frac{  k(\bx + \zeta {\bm{e}}_i, \bx + \zeta {\bm{e}}_i) -  k(\bx + \zeta {\bm{e}}_i, \bx  )    -k(\bx , \bx + \zeta {\bm{e}}_i)     +k(\bx,\bx)        }{\zeta^2} \nonumber \\
&= 
\{\sigma^{(1)}_{0,i} (\bx) \} ^2 
+ \frac{1}{2}  \left . \frac{\partial  ^3  k({\bm{a}}, \bx^\ast)}{\partial a_i \partial a_i \partial x^\ast_i  } \right |_{\bx^\ast = \bx , {\bm{a}}=\bx^\star_3} \zeta  
+  \frac{1}{2} 
\left . \frac{\partial ^3  k(\tilde{\bx}, \bx^\ast)}{\partial \tilde{x}_i \partial x^\ast_i \partial x^\ast_i } \right |_{\bx^\ast = \bx , \tilde{\bx} = \bx^\star_4} \zeta \nonumber \\
& \q +   \frac{1}{6}   \left . \frac{\partial ^3  k(\bx + \zeta {\bm{e}}_i, \bx^\ast)}{\partial x^\ast_i \partial x^\ast_i \partial x^\ast_i} \right |_{\bx^\ast = \bx^\star_2 } \zeta -
 \frac{1}{6}   \left . \frac{\partial ^3  k(\bx, \bx^\ast)}{\partial x^\ast_i \partial x^\ast_i \partial x^\ast_i} \right |_{\bx^\ast = \bx^\star_1 } \zeta . \label{eq:ap4-D9}
\end{align}
Finally, by substituting \eqref{eq:ap4-D5}, \eqref{eq:ap4-D6} and \eqref{eq:ap4-D9} into \eqref{eq:ap4-D1}, we get 
\begin{align*}
\V[\tilde{f}^{(1)}_{0,i} (\bx ; \zeta) ]  = \zeta & \left (
\frac{1}{2}  \left . \frac{\partial  ^3  k({\bm{a}}, \bx^\ast)}{\partial a_i \partial a_i \partial x^\ast_i  } \right |_{\bx^\ast = \bx , {\bm{a}}=\bx^\star_3} +
\frac{1}{2} 
\left . \frac{\partial ^3  k(\tilde{\bx}, \bx^\ast)}{\partial \tilde{x}_i \partial x^\ast_i \partial x^\ast_i } \right |_{\bx^\ast = \bx , \tilde{\bx} = \bx^\star_4} \right . \\
&\ +\left . 
\frac{1}{6}   \left . \frac{\partial ^3  k(\bx + \zeta {\bm{e}}_i, \bx^\ast)}{\partial x^\ast_i \partial x^\ast_i \partial x^\ast_i} \right |_{\bx^\ast = \bx^\star_2 }   -
 \frac{1}{6}   \left . \frac{\partial ^3  k(\bx, \bx^\ast)}{\partial x^\ast_i \partial x^\ast_i \partial x^\ast_i} \right |_{\bx^\ast = \bx^\star_1 }   \right . \\
&\  \left . 
- \frac{1}{2} \left . \frac{\partial ^3   k({\bm{a}}, \bx^\ast )  }{\partial  a_i  \partial x^\ast_i \partial x^\ast_i } \right |_{{\bm{a}}= \bx, \bx^\ast = \bx^\star_5 }
-  \frac{1}{2} \left . \frac{\partial ^3   k( \bx^\ast,{\bm{a}} )  }{\partial x^\ast_i \partial x^\ast_i  \partial  a_i  } \right |_{  \bx^\ast = \bx^\star_6,            {\bm{a}}= \bx } 
\right ).
\end{align*}
Note that $k(\bx,\bx^\prime)$ is a five times continuously differentiable function and $D$ is a compact set. 
This implies that there exists a positive constant $C_0$ such that 
$
|\V[\tilde{f}^{(1)}_{0,i} (\bx ; \zeta) ] |  \leq |\zeta| C_0. 
$ 
Similarly, using same argument the inequality $|\V[\tilde{f}^{(2)}_{0,jk} (\bx ; \zeta) ] |  \leq |\zeta| C_0$ also holds. 
\end{proof}

By replacing $\mathcal{X}$ with $\mathcal{X}^\star$, we obtain the following corollary.
\begin{cor}\label{cor:sc2}
Let $D$ be a compact set, and let ${\rm int} (D) $ be an interior set of $D$. 
Suppose that $\mathcal{X} ^\star$ is a finite subset of ${\rm int} (D) $. 
Assume that the kernel function $k(\bx,\bx^\prime)$ is a five times continuously differentiable function. 
Then, the assumptions (C\ref{enu:C2}) and (C\ref{enu:C11}) hold.
\end{cor}

Finally, the following lemma holds:
\begin{lemma}
Let $D$ be a compact set. 
Assume that the kernel function $k(\bx,\bx')$ is an eight times differentiable function. Then, the assumption (C\ref{enu:C1}) holds.
\end{lemma}
\begin{proof}
For GP samples $g$, from Theorem 5 of  
Ghosal and Roy \cite{ghosal2006posterior},   if a kernel function $\tilde{k} (\bx , \bx')$ of $g$  has a fourth derivative, there exists positive 
constants $a$ and $b$ such that 
$$
\PR (   \sup_{\bx \in D} | \partial g / \partial x_l | >L )   \leq a e^{-b  L^2} .
$$
Here, for the GP sample $f^{(1)}_{i} $, its kernel function  is the second derivative of the kernel function $k(\bx,\bx')$. 
Therefore, if   $k(\bx,\bx')$ has a sixth derivative, then the kernel function of $f^{(1)}_{i} $ has a fourth derivative. 
Similarly, if $k(\bx,\bx')$ has an eighth derivative, then the kernel function of $f^{(2)}_{jk} $ also has  a fourth derivative. 
\end{proof}

  \end{document}